\documentclass[10pt,twocolumn,letterpaper]{article}
\usepackage{cvpr} %
\usepackage[accsupp]{axessibility}
\usepackage{amsthm}
\usepackage{stfloats}
\usepackage{multirow}
\usepackage{array}
\usepackage{enumitem}
\usepackage{colortbl}
\usepackage{makecell}
\usepackage{threeparttable}
\newtheorem{theorem}{Theorem}
\newtheorem{observation}{Observation}
\usepackage{xcolor}
\newcommand{\red}[1]{{\color{red}#1}}

\def\ie{{\textit{i}.\textit{e}.}}
\def\etal{{{et al}.}}
\newcommand\Bstrut{\rule[-0.9ex]{0pt}{0pt}}   %
\usepackage{marvosym}

\definecolor{cvprblue}{rgb}{0.21,0.49,0.74}
\usepackage[pagebackref,breaklinks,colorlinks,citecolor=cvprblue]{hyperref}

\def\paperID{1554} 
\def\confName{CVPR}
\def\confYear{2024}

\title{Defense Against Adversarial Attacks on No-Reference Image Quality Models with Gradient Norm Regularization}

\author{
Yujia Liu$^{1,2}$\footnotemark[1], Chenxi Yang$^{3,1}$\footnotemark[1], Dingquan Li$^4$, Jianhao Ding$^{1,2}$, Tingting Jiang$^{1,2}$\textsuperscript{\Letter} \\
$^1$NERCVT, School of Computer Science, Peking University, China \\
$^2$National Key Laboratory for Multimedia Information Processing, Peking University, China  \\
$^3$School of Mathematical Sciences, Peking University, China
$^4$Peng Cheng Laboratory, China \\
{\tt\small \{yujia\_liu,dingquanli,ttjiang\}@pku.edu.cn; \{yangchenxi,djh01998\}@stu.pku.edu.cn}
}

\begin{document}
\maketitle
\renewcommand{\thefootnote}{\fnsymbol{footnote}}
\footnotetext[1]{Equal contribution}
\begin{abstract}
The task of No-Reference Image Quality Assessment (NR-IQA) is to estimate the quality score of an input image without additional information. NR-IQA models play a crucial role in the media industry, aiding in performance evaluation and optimization guidance. However, these models are found to be vulnerable to adversarial attacks, which introduce imperceptible perturbations to input images, resulting in significant changes in predicted scores. In this paper, we propose a defense method to improve the stability in predicted scores when attacked by small perturbations, thus enhancing the adversarial robustness of NR-IQA models. To be specific, we present theoretical evidence showing that the magnitude of score changes is related to the $\ell_1$ norm of the model's gradient with respect to the input image. Building upon this theoretical foundation, we propose a norm regularization training strategy aimed at reducing the $\ell_1$ norm of the gradient, thereby boosting the robustness of NR-IQA models. Experiments conducted on four NR-IQA baseline models demonstrate the effectiveness of our strategy in reducing score changes in the presence of adversarial attacks. To the best of our knowledge, this work marks the first attempt to defend against adversarial attacks on NR-IQA models. Our study offers valuable insights into the adversarial robustness of NR-IQA models and provides a foundation for future research in this area.
\end{abstract}
    
\vspace{-0.2in}
\section{Introduction}\label{sec:intro}

\begin{figure}[!t]
    \centering
    \includegraphics[width=0.9\linewidth]{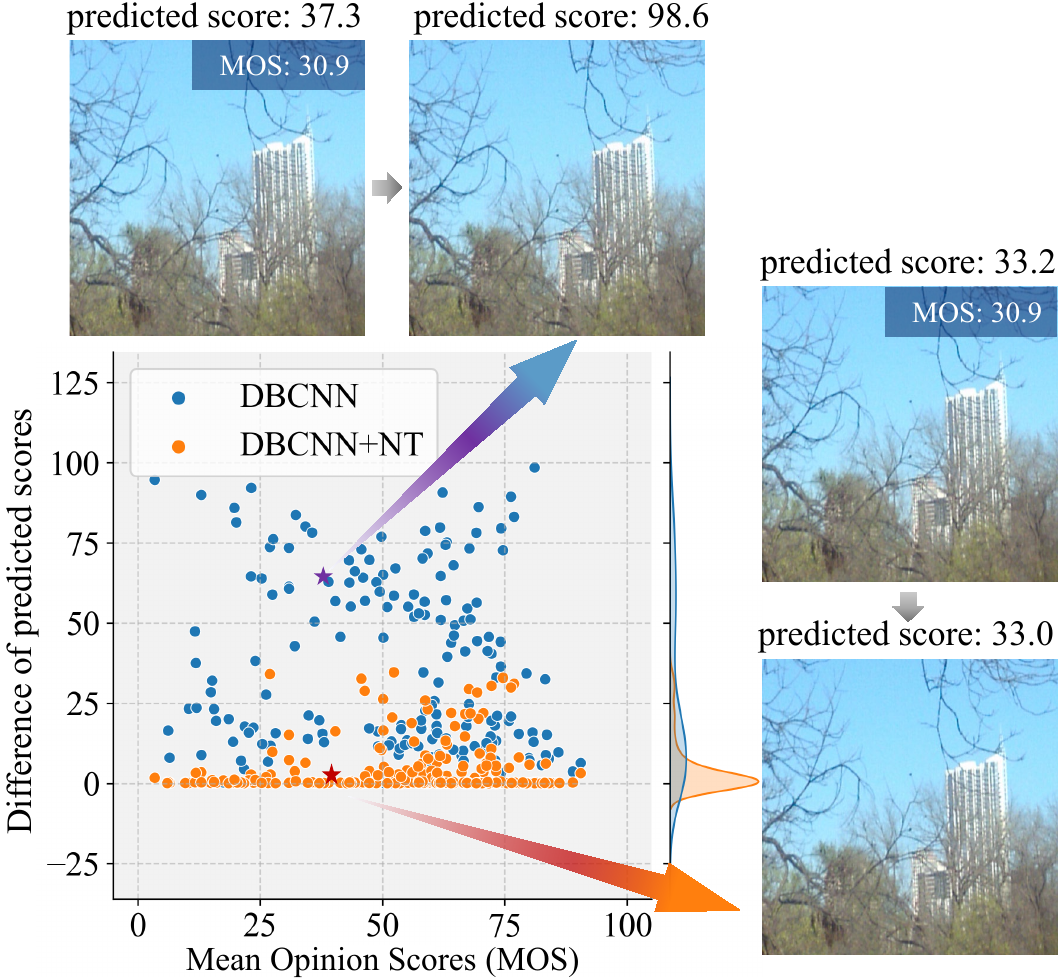}
    \caption{Comparison of DBCNN~\cite{2020_TCSVT_DBCNN} trained with and without the proposed Norm regularization Training (NT) strategy under the Perceptual Attack~\cite{2022_NIPS_Zhang_PAttack} using the same setting. The absolute differences between predicted scores before and after the attack ($\vert s_\text{after}-s_\text{before}\vert$) for all test images are presented, with the fitted distribution displayed on the right side of the picture. An example is shown with predicted scores before and after the attack (zoom in for a better view). It is evident that DBCNN+NT exhibits smaller score changes compared to the baseline model.}
    \label{intro_sca}
\end{figure}
\setlength{\textfloatsep}{2pt}

Deep Neural Networks (DNNs) have demonstrated remarkable performance across various domains~\cite{2016_CVPR_He_ResNet,2021_CompSur_Khan_AutoDrive,2021_NeuroComput_Mishra_Diagnosis}, and Image Quality Assessment (IQA) is one of them. IQA aims to predict the quality of images consistent with human perception. And it could be categorized as Full-Reference~(FR) and No-Reference~(NR) according to the access to the reference images. While FR-IQA models specialize in assessing the perceptual disparities between two images, NR-IQA models focus on estimating a quality score for a single input image. The importance of IQA extends to many applications such as image transport systems~\cite{2023_TCSVT_Fu_Compression}, 
image in-painting~\cite{2019_IJCV_Mariko_Inpaint} and so on~\cite{2017_MM_Ma_VideoQual,2019_ICCV_SuperResolut,2019_ICCV_Oren_VideoComp}. %
Leveraging the capabilities of DNNs, recent IQA models have achieved remarkable consistency with human opinion scores~\cite{2020_CVPR_hyperIQA}.

However, the reliability of DNNs is challenged since they are found to be susceptible to adversarial perturbations. Attackers would mislead DNNs to make decisions inconsistent with human perception by adding carefully designed perturbations to inputs. This manipulation technique is called the adversarial attack, and the perturbed inputs are called adversarial examples. The initial discovery of DNNs' vulnerability to adversarial attacks was in the context of classification tasks~\cite{2014_ICLR_Goodfellow_ProposeAdv}. Subsequently, the threats of adversarial attacks are explored in various tasks, including object detection~\cite{2019_CVPR_Simen_AdvPatch}, segmentation~\cite{2019_MICCAI_Utku_AttackImageSegmentation}, natural language processing~\cite{2020_ICLR_Zhu_AdvInNLP}, and many others~\cite{2022_ICLR_Hadi_TransferInASRs,2021_ICML_Liang_ConnectionsBetweenAdvTransAndKnowledgeTrans}. %

Recently, adversarial attacks on IQA models have garnered significant attention. Several attack methods targeting NR-IQA models have been proposed, where attackers aim to significantly change the predicted scores with small adversarial perturbations to input images. For instance, \citet{2022_NIPS_Zhang_PAttack} generated adversarial examples using the Lagrange multiplier method, imposing several constraints on the quality of adversarial examples. 
Besides, Shumitskaya~\etal~\cite{2022_BMVC_Ekarerina_UAP,2023_ICLRt_Ekaterina_FACPA} and Korhonen~\etal~\cite{2022_QEVMAw_Korhonen_BIQA} trained an individual model to generate adversarial examples.

However, despite these proposed attack techniques highlighting vulnerabilities in NR-IQA models, no methods have been put forth to defend against attacks and improve the adversarial robustness of NR-IQA methods. Training robust IQA models is essential for improving the reliability of these models in real-world applications. 
For instance, in online advertising, the quality of ad images can significantly impact viewer engagement. Adversarial attacks on NR-IQA metrics could result in low-quality ad images being rated highly or high-quality ad images being rated lowly, as cases shown in Figure~\ref{fig:score_based_examples}, potentially reducing the effectiveness of online advertising campaigns.
Therefore, there is an impending need to train robust NR-IQA models, which is crucial for ensuring both the reliability and security of NR-IQA models in practical applications.

\begin{figure*}[!t]
    \centering
    \includegraphics[width=0.85\linewidth]{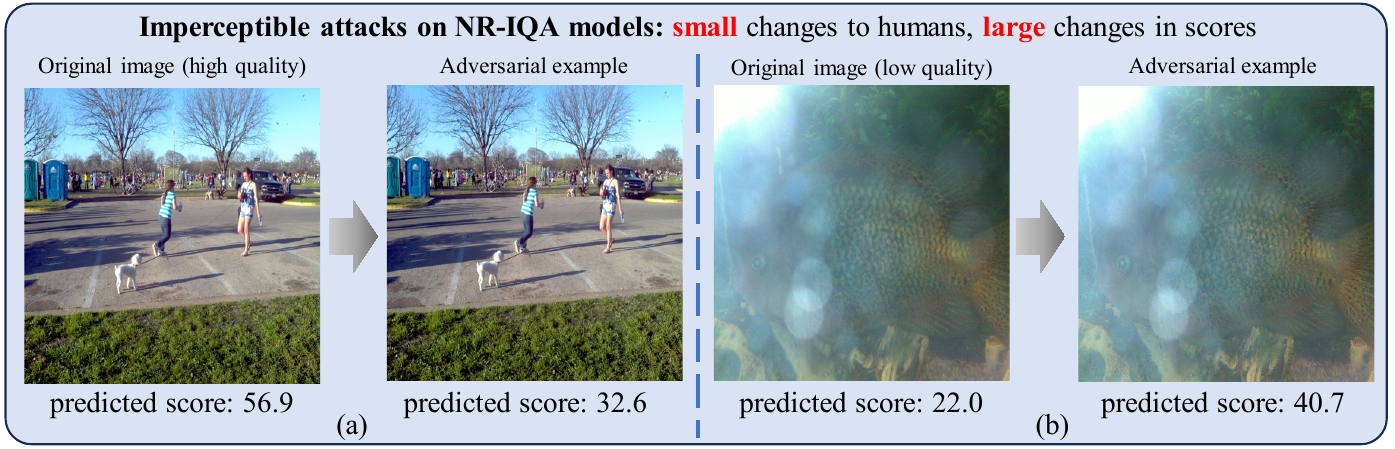}
    \caption{(Zoom in for a better view) Examples of adversarial attacks on the DBCNN~\cite{2020_TCSVT_DBCNN} model. The range of MOS is $[0,100]$.}
    \label{fig:score_based_examples}
\end{figure*}

In this paper, we propose a defense method to improve the adversarial robustness of NR-IQA models in terms of reducing the quality score changes before and after adversarial attacks, which is supported by both theoretical foundations and empirical evidence.
We analyze existing attacks on NR-IQA models
and establish a theoretical foundation demonstrating the strong relationship between the adversarial robustness of an NR-IQA model $f$ and the $\ell_1$ norm of its gradient $\nabla_x f(x)$ concerning the input image $x$. We found that for an NR-IQA model, a smaller $||\nabla_x f(x)||_1$ implies a more robust model.
Drawing upon the theoretical analysis, we propose the regularization of the gradient's $\ell_1$ norm to enhance the adversarial robustness of NR-IQA models. A direct way to regularize $\Vert \nabla_x f(x) \Vert_1$ is adding it to the loss function in the training phase, which needs double backpropagation to compute the gradient of the regularization term with respect to model parameters considering the calculation mechanism of DNNs. However, double backpropagation is not currently scalable for large-scale DNNs~\cite{2019_ICLR_Ilyas_EstimateGradient}.
Therefore, we approximate $\Vert \nabla_x f(x) \Vert_1$ by finite differences~\cite{finite_difference} instead of using it directly. The approximation result is utilized as the regularization term, effectively constraining $\Vert \nabla_x f(x) \Vert_1$. 

To further verify our methodology,
we conduct experiments on four baseline NR-IQA models and four attack methods. The results show the effectiveness of the norm regularization strategy in boosting baseline models' robustness against adversarial attacks. %
To the best of our knowledge, this is the first work to propose a defense method against adversarial attacks on NR-IQA models, which uses the $\ell_1$ norm of the gradient as a regularization term. This paper establishes a theoretical connection between the robustness of the NR-IQA model against adversarial attacks and the gradient norm with respect to the input image.
To support reproducible scientific research, we release the code at \url{https://github.com/YangiD/DefenseIQA-NT}.

\vspace{-0.1in}
\section{Related Work}
Adversarial attacks were first studied in classification tasks, so we introduce these attacks along with defense methods in Sec.~\ref{sec:attack_class}. Sec.~\ref{sec:NR_IQA} and Sec.~\ref{sec:attack_IQA} provide a brief overview of NR-IQA models and attacks on these models.

\subsection{Adversarial Attacks and Their Defenses in Classification Tasks}\label{sec:attack_class}
Based on the available knowledge of the target model, adversarial attacks can be divided into white-box attacks and black-box attacks. In white-box scenarios, attackers possess comprehensive knowledge about the target model.
Some classic attacks treated the problem of generating adversarial examples as an optimization task~\cite{2015_ICLR_Goodfellow_FGSM,2017_SP_Carlini_CW,2016_CVPR_Moosavi_DeepFool}.
Alternatively, some attacks proposed to train a new model to generate adversarial examples~\cite{2018_ICLR_Zhao_NaturalGAN,2018_IJCAI_Xiao_AdvGAN,2019_AAAI_Liu_PS-GAN}.
Conversely, in the case of black-box attacks, attackers are restricted to accessing only the output of the target model.
A predominant strategy for executing black-box attacks involves generating adversarial examples on a known source model and subsequently transferring them to the target model~\cite{2019_CVPR_Cihang_DI-FGSM,2019_CVPR_YiranChen_AA,2022_CompSec_Liu_LowFreTransAdv}.

To defend against adversarial attacks in classification tasks, a widely used method is adversarial training~\cite{2014_ICLR_Goodfellow_ProposeAdv} and its variants~\cite{2018_ICLR_Goodfellow_EnsembleAttack,2019_NIPS_Zhang_YOPO,2019_NIPS_Ali_FreeAT,2020_ICLR_Zhu_AdvInNLP}. Adversarial training involves the generation of adversarial examples using specific attacks, which are incorporated into the training dataset so that the model can learn from these adversarial examples in the training phase. It acts as a form of data augmentation and helps to improve the robustness of the network. 

\subsection{IQA Tasks and Models}\label{sec:NR_IQA}
IQA tasks aim to predict image quality scores consistent with human perception (\ie, Mean Opinion Score, MOS for short), which could be divided into FR and NR. For FR-IQA, it involves comparing a distorted image and its reference image to predict the quality score of the distorted image. Due to the difficulty of obtaining reference images in some authentic scenes, NR-IQA proposes to predict the quality score with only the distorted image.

NR-IQA methods extract features related to human perception of image quality. 
Some methods~\cite{DIIVINE_2011, BRISQUE_2012} considered the hand-craft feature from Natural Scene Statistics. Further works explored the impact of image semantic information on human perception of image quality. HyperIQA~\cite{2020_CVPR_hyperIQA} used a hypernetwork to obtain different quality estimators for images with different content. DBCNN~\cite{2020_TCSVT_DBCNN} extracted distorted information and semantic information of images by two independent neural networks and combined them with a bi-linear pooling. LinearityIQA~\cite{2020_MM_LinearityIQA} proposed the normalization of scores in the loss function for faster convergence of the model. 
Meanwhile, some methods explored the effectiveness of different network architectures. MANIQA~\cite{2022_CVPRw_MANIQA} and MUSIQ~\cite{Ke_2021_ICCV_MUSIQ} utilized vision transformers~\cite{dosovitskiy2020vit} and verified their effectiveness in NR-IQA tasks.

\subsection{Adversarial Attacks on NR-IQA Models}
\label{sec:attack_IQA}
The issue of adversarial attacks within the context of IQA tasks has garnered some attention, although research in this area remains somewhat limited.
Recently, some attack methods have been designed for NR-IQA models.

In white-box scenarios, the Perceptual Attack~\cite{2022_NIPS_Zhang_PAttack} modeled NR-IQA attacks as an optimization problem, where it employed the Lagrange multiplier method to solve this optimization problem. Perceptual Attack had tried different constraints on the image quality of adversarial examples, including the Chebyshev distance, LPIPS~\cite{2018_CVPR_Richard_LPIPS}, SSIM~\cite{2004_TIP_Wang_SSIM} and DISTS~\cite{2020_TPAMI_Ding_DISTS}. \citet{2022_BMVC_Ekarerina_UAP} proposed to update a universal perturbation through a set of images and added it to clean images to attack NR-IQA models.

In black-box scenarios, the Kor. Attack~\cite{2022_QEVMAw_Korhonen_BIQA} adapted ideas from attacks in classification tasks, creating adversarial examples through ResNet50~\cite{2016_CVPR_He_ResNet} and transferring them to the unknown target model.
Likewise, 
\citet{2023_ICLRt_Ekaterina_FACPA} proposed to train a U-Net~\cite{2015_MICCAI_Olaf_UNet} to generate different adversarial perturbations for each image. 

Although Korhonen~\etal~\cite{2022_QEVMAw_Korhonen_BIQA} adapted basic defense mechanisms from classification models to NR-IQA models, they did not investigate defense methods that are specifically designed for NR-IQA tasks.

\renewcommand{\thefootnote}{\arabic{footnote}}
\section{Preliminary}

\subsection{Definition of Attacks on NR-IQA Models}\label{sec:method_definition}
Adversarial attacks on NR-IQA models aim to manipulate the predicted score of an input image $x$ by an NR-IQA model $f$ so that the objective score by models is inconsistent with the subjective score by humans.
As for a successful attack on an NR-IQA model, there are imperceptible differences to the human eye between original images and adversarial examples, but these subtle perturbations result in large changes in the predicted scores generated by NR-IQA models. Examples are shown in Figure~\ref{fig:score_based_examples} (a) and (b). This attack can be mathematically described as follows:
\begin{equation}
\label{eq:attack_definition}
    \max \ |f(x+\delta) - f(x)|, \ \text{s.t.} \ D(x+\delta,x) \leqslant \epsilon, %
\end{equation}
where $\delta$ symbolizes the perturbation added to $x$, while the function $D(\cdot,\cdot)$ quantifies the perceptual distance between two images, and $\epsilon$ characterizes the tolerance of human eyes for image differences.
An assumption is that when $\ D(x+\delta,x) \leqslant \epsilon$, the subjective score of $x+\delta$ is the same as $x$.
In our methodology, we take $D(\cdot,\cdot)$ as defined by:
\begin{equation}
    D(x+\delta, x) = \Vert \delta \Vert_\infty,
\end{equation}
for the convenience of our theoretical analysis.
Moreover, it has been used in attacks for IQA tasks~\cite{2022_NIPS_Zhang_PAttack,2022_QEVMAw_Korhonen_BIQA}.

\subsection{Robustness Evaluations of NR-IQA Models}\label{sec:metric}
The adversarial robustness of NR-IQA measures the stability of the NR-IQA model to imperceptible perturbations of input images generated by attacks. For example, when the original image and its adversarial example have the same appearance, 
an NR-IQA model should give both the same quality scores.
Figure~\ref{fig:score_based_examples} presents instances where DBCNN fails in this aspect.
Researchers tend to assess the adversarial robustness of NR-IQA models by evaluating their IQA performance on adversarial examples~\cite{2022_QEVMAw_Korhonen_BIQA,2022_NIPS_Zhang_PAttack}.
A better performance implies that the model is more robust.

Typically, the performance of an NR-IQA model is measured using four metrics: Root Mean Square Error (RMSE), Pearson's Linear Correlation Coefficient (PLCC), Spearman Rank-Order Correlation Coefficient (SROCC), and Kendall Rank-Order Correlation Coefficient (KROCC).\footnote{Formulations of these metrics are in the supplementary material.} RMSE and PLCC are indicators of prediction accuracy, while SROCC and KROCC assess the prediction monotonicity~\cite{VQEG_2000}. When an NR-IQA model is attacked, greater robustness is indicated by smaller RMSE and larger PLCC, SROCC, and KROCC values. In this paper, we provide a theoretical analysis of robustness in terms of RMSE, and test all these metrics in the experimental part.

Recently, some IQA-specific metrics were proposed for evaluating the robustness of NR-IQA models~\cite{2022_NIPS_Zhang_PAttack,antsiferova2023comparing}. We will discuss these metrics in the supplementary material.

\section{Methodology}\label{sec:method}
In this section, we offer a theoretical exposition on improving NR-IQA models' adversarial robustness in terms of the magnitude of changes in predicted scores. We show that the robustness can be enhanced by regularizing the $\ell_1$ norm of the gradient. We also propose a method for training a robust NR-IQA model using the norm regularization method.

\subsection{Why to Regularize Gradient Norm?}\label{sec:method_why}
In this subsection, we will outline the theoretical foundations regarding the relationship between the robustness in terms of score changes and the $\ell_1$ norm of the gradient. It raises the necessity of regularizing the $\ell_1$ norm of the input gradient of the predicted score.
We prove that the magnitude of changes in predicted scores can be effectively approximated by the $\ell_1$ norm of $\nabla_x f(x)$, with the assumption that the $\ell_\infty$ norm of perturbations is bounded.

\begin{theorem}
\label{thr:l1}
Suppose $f$ represents an NR-IQA model, $\epsilon$ is the strength of an attack, and $x$ denotes an input image. The maximum change in predicted scores of $x$ by $f$ against $\ell_\infty$-bounded attacks is highly correlated to $\Vert \nabla_x f(x) \Vert_1$, which can be formulated as
\begin{equation}\label{eq:theorem1}
    \sup_{\delta:\Vert \delta \Vert_\infty \leqslant \epsilon} |f(x+\delta) - f(x)| \approx \epsilon \Vert \nabla_x f(x)\Vert_1.
\end{equation}
\end{theorem}

\begin{proof}
To begin, we apply the first-order Taylor expansion to the function $f(x+\delta)$ in the vicinity of $x$, yielding:
\begin{equation}
    f(x+\delta) \approx f(x) + \delta^T \nabla_x f(x).
\end{equation}
Then, $ |f(x+\delta) - f(x)| \approx |\delta^T \nabla_x f(x)|$. Meanwhile, $|\delta^T \nabla_x f(x)|$ has the maximum value when $\delta = \epsilon\cdot \text{sign}(\nabla_x f(x))$, and this leads to Eq.~\eqref{eq:theorem1}.
\end{proof}

This theorem establishes the connection between changes in predicted scores and the $\ell_1$ norm of the gradient.
According to Theorem~\ref{thr:l1}, suppose the strength of attacks $\epsilon$ is fixed, then the extent of score changes is primarily determined by the $\ell_1$ norm of the gradient $\nabla_x f(x)$.
In practical terms, this signifies that the regularization of $\Vert \nabla_x f(x) \Vert_1$ will lead to smaller fluctuations in predicted scores and thereby improve the adversarial robustness of $f$ against imperceptible attacks.

\subsection{How to Regularize Gradient Norm?}\label{sec:method_how}
To train a robust NR-IQA model incorporating gradient norm regularization, a direct way is to add the $\ell_1$ norm of gradients to the loss function, \ie,
\begin{equation}
    L(f,x) = L_{\text{IQA}}(f,x) + \lambda \cdot \Vert \nabla_x f(x) \Vert_1^2.
\end{equation}
The loss function $L(f,x)$ comprises two components: the loss $L_{\text{IQA}}$ tailored to the specific NR-IQA task, and the norm regularization term with a positive weight $\lambda$.

However, directly adding the term $\Vert \nabla_x f(x) \Vert_1$ to the loss function leads to the requirement of double backpropagation for computing the gradient of this term with respect to model parameters, which is time-consuming and currently not suitable for large-scale DNNs~\cite{2021_Chris_GradientApproximation}. Therefore, we employ an approximation technique for the regularization term. Drawing inspiration from the methodology presented in the work~\cite{2019_ICLR_Ilyas_EstimateGradient}, we leverage the finite difference~\cite{finite_difference} technique to estimate $\Vert \nabla_x f(x) \Vert_1$, \ie,
\begin{equation}
    \Vert \nabla_x f(x) \Vert_1 \approx 
    \left| \frac{
        f(x+h\cdot d) - f(x)
        }{
        h
        } \right|,
\label{eq:approx_l1}
\end{equation}
where $h\in\mathbb{R}^{+}$ is the step size and $d=\text{sign}(\nabla_x f(x))$.

Finally, the loss function with the regularization of the $\ell_1$ norm of the gradient is as follows:
\begin{equation}
    L(f,x) = L_{\text{IQA}}(f,x) + \lambda \cdot \left| \frac{
        f(x+h\cdot d) - f(x)
        }{
        h
        } \right|^2.
\label{eq:training_loss}
\end{equation}

\section{Experiments}
\label{sec:exper}
In this section, we present extensive experiments conducted on various NR-IQA baseline models to validate the efficacy of our proposed Norm regularization Training (NT) strategy.
We briefly overview our experimental setup in Sec.~\ref{sec:exper_set}. Subsequently, in Sec.~\ref{sec:ex_robust_performance}, we demonstrate the enhancement in robustness achieved by the NT strategy against a diverse set of attacks.
Furthermore, we illustrate the role of the finite difference approximation in reducing the $\ell_1$ norm of the gradients (Sec.~\ref{sec:ex_norm}), as well as the relationship between attack intensity and robustness (Sec.~\ref{sec:ex_intensity}). We also perform ablation studies on hyperparameters $\lambda$ and $h$ in Sec.~\ref{sec:ex_ablation}. 

\subsection{Experimental Settings}
\label{sec:exper_set}
Experiments were carried out on the popular LIVEC dataset~\cite{2015_TIP_LIVEC}. We randomly selected 80\% of the images for training and the remaining 20\% for testing and attacks.

Our experiments to assess the robustness of NR-IQA models are structured along three key dimensions, as depicted in Figure~\ref{fig:exper_setting}. The first dimension revolves around the choice of the baseline models. We evaluate our NT strategy on four prominent NR-IQA baseline models: HyperIQA~\cite{2020_CVPR_hyperIQA}, DBCNN~\cite{2020_TCSVT_DBCNN}, LinearityIQA~\cite{2020_MM_LinearityIQA}, and MANIQA~\cite{2022_CVPRw_MANIQA}. Each of these baseline models is referred to as ``baseline,'' while the models trained with NT are denoted as ``baseline+NT.'' The NT strategy is applied to HyperIQA with the weight $\lambda=0.001$, DBCNN, and LinearityIQA with $\lambda=0.0005$ and MANIQA with $\lambda=0.003$.
For all models, the step size $h=0.01$.
Further training settings are provided in the supplementary material.

\begin{figure}[!th]
    \centering
    \includegraphics[width=0.95\linewidth]{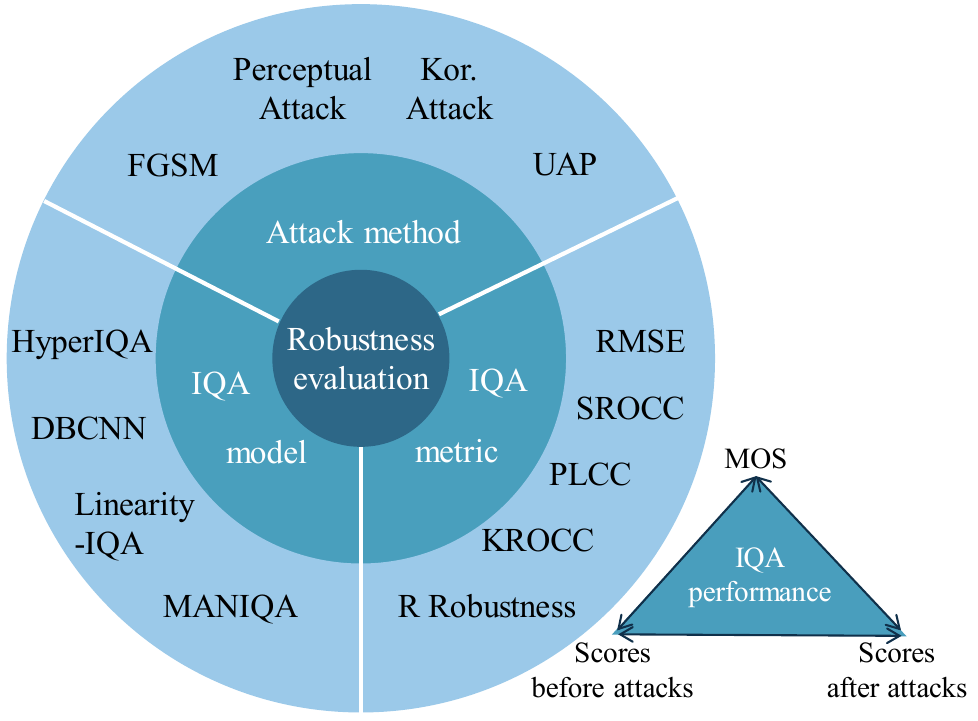}
    \caption{Three dimensions in experimental settings.}
    \label{fig:exper_setting}
\end{figure}

The second dimension involves the selection of attack methods. We employ four attack methods designed for NR-IQA tasks. These attacks include two white-box attacks: FGSM\footnote{FGSM was originally designed for classification tasks, but we modify its loss to for NR-IQA tasks (refer to the supplementary material).}~\cite{2015_ICLR_Goodfellow_FGSM} and Perceptual Attack~\cite{2022_NIPS_Zhang_PAttack}, as well as two black-box attacks: UAP\footnote{UAP is proposed as a white-box attack. We employ its perturbation generated on PaQ-2-PiQ model~\cite{Ying_2020_CVPR}, and serve UAP as a black-box attack.}~\cite{2022_BMVC_Ekarerina_UAP} and Kor. Attack~\cite{2022_QEVMAw_Korhonen_BIQA}.
To ensure fairness in our evaluations, each attack method uses the same setting (\ie, employing the same hyperparameters in attack) when targeting different models. We set different hyperparameters for different attacks and ensure the majority of attacked images’ SSIM~\cite{2004_TIP_Wang_SSIM} was above 0.9 to satisfy the assumption that the MOS is the same for both images before and after the attack. 
Detailed hyperparameter information for these attacks can be found in the supplementary material. 
Black-box attacks are indicated with an asterisk in the tables of this paper.

The third dimension pertains to the evaluation metrics for NR-IQA models. 
As detailed in Sec.~\ref{sec:metric}, we follow the evaluations in previous works and consider four metrics in this paper, \ie, RMSE, PLCC, SROCC, and KROCC.
Additionally, we incorporate the $R$ robustness~\cite{2022_NIPS_Zhang_PAttack} into our analysis. This metric is proposed to assess the model robustness by measuring the relative score changes before and after attacks.
Formulations of these metrics are shown in the supplementary material.
Except for $R$ robustness, other metrics are conventionally computed by comparing the predicted scores by models against MOS provided by humans.\footnote{We normalize MOS to $[0,100]$ for a straightforward comparison.} 
In our evaluation, we extend the analysis of these metrics to assess predicted scores by a model both before and after attacks, since attackers only possess prediction scores before attacks.
Notably, in accordance with the attack definition in Eq.~\eqref{eq:attack_definition}, the NT strategy primarily focuses on the magnitude of changes in predicted scores between predicted scores before and after attacks. It could be measured by RMSE between predicted scores before and after attacks.

\subsection{Robustness Improvement}\label{sec:ex_robust_performance}
In this subsection, we present the performance of NR-IQA models on unattacked images 
(where R robustness is not applicable),
as well as their adversarial robustness against different attack methods.
Our experimental results are summarized into four key observations.
We provide additional analysis of adversarial robustness improvement in the supplementary material.

\begin{table}[!th]
\caption{Performance of NR-IQA models on unattacked images (``baseline $\big /$ baseline+NT''). \textbf{Bold} denotes better value in a cell.}
\centering
\renewcommand\arraystretch{1.2}
\resizebox{\hsize}{!}{
\begin{tabular}{lcccc}
\toprule
       & {\makecell[c]{HyperIQA \\ base / +NT }}            & {\makecell[c]{DBCNN \\ base / +NT }}    & {\makecell[c]{LinearityIQA \\ base / +NT }}    & {\makecell[c]{MANIQA \\ base / +NT }} \\ \midrule
RMSE$\downarrow$  & \textbf{9.913} $\big/$ 12.575& \textbf{10.897} $\big/$ 13.140 & \textbf{12.730} $\big/$ 13.173 &    26.082   $\big/$    \textbf{23.830}    \\
SROCC$\uparrow$ & \textbf{0.899 }$\big/$ 0.859& \textbf{0.866} $\big/$ 0.856  & \textbf{0.832} $\big/$ 0.820   &    \textbf{0.876}$\big/$   0.871       \\
PLCC$\uparrow$  & \textbf{0.916} $\big/$ 0.868& \textbf{0.892} $\big/$ 0.849  & \textbf{0.840} $\big/$ 0.827   &   0.870  $\big/$   \textbf{0.876}      \\
KROCC$\uparrow$ & \textbf{0.728} $\big/$ 0.670& \textbf{0.688} $\big/$ 0.666  & \textbf{0.641} $\big/$ 0.627   &   \textbf{0.696}  $\big/$  0.692   \\
\bottomrule
\end{tabular}
}
\label{tab:clean_image}
\end{table}

\begin{observation}
    The NT strategy results in a slight decrease in the performance of NR-IQA models on clean images.
\end{observation}

The performance of both the baseline models and their NT-enhanced versions on unattacked images are shown in Table~\ref{tab:clean_image}.
These metrics are calculated between MOS values and predicted scores on unattacked images.
We can see that the NT strategy leads to a slight decrease in RMSE, SROCC, PLCC, and KROCC compared with baseline models.
Similar trends were reported in the context of classification models that defense methods would cause a decline in classification accuracy on clean images.~\cite{2014_ICLR_Goodfellow_ProposeAdv,2018_ICLR_Goodfellow_EnsembleAttack,2019_NIPS_Zhang_YOPO}.
These findings suggest that enhanced robustness is often achieved at the cost of reducing performance on unattacked images.
Nonetheless, the performance decline induced by the NT strategy is marginal and well within acceptable limits.

\begin{table*}[!t]
\caption{The \textbf{RMSE$\downarrow$} metric of NR-IQA models against attacks (``baseline $\big /$ baseline+NT''). \textbf{Bold} denotes better value in a cell.}
\centering
\renewcommand\arraystretch{1.3}
\resizebox{\textwidth}{!}{
\begin{tabular}{lcccc|cccc}
\toprule
  &
  \multicolumn{4}{c}{MOS \& Predicted Score After Attack} &
  \multicolumn{4}{c}{Score Before Attack   \& Score After Attack} \\ \cmidrule(lr){2-5} \cmidrule(lr){6-9}
  & {\makecell[c]{HyperIQA \\ base / +NT }}            & {\makecell[c]{DBCNN \\ base / +NT }}    & {\makecell[c]{LinearityIQA \\ base / +NT }}    & {\makecell[c]{MANIQA \\ base / +NT }}  & {\makecell[c]{HyperIQA \\ base / +NT }}            & {\makecell[c]{DBCNN \\ base / +NT }}    & {\makecell[c]{LinearityIQA \\ base / +NT }}    & {\makecell[c]{MANIQA \\ base / +NT }}  \Bstrut\\ 
  \hline
FGSM &
  25.729 $\big/$ \textbf{16.828}&
  36.758 $\big/$ \textbf{24.711} &
  50.823 $\big/$ \textbf{40.104} &
  \textbf{24.899} $\big/$ 25.712&
  19.174 $\big/$ \textbf{7.885} &
  32.778 $\big/$ \textbf{19.065} &
  48.128 $\big/$ \textbf{36.988} &
   15.549 $\big/$ \textbf{6.562}\\
Perceptual & 
13.565 $\big/$ \textbf{12.593}&
88.864 $\big/$ \textbf{51.961}    &
115.395 $\big/$ \textbf{80.949}   &
22.745 $\big/$ \textbf{21.998}   &
6.360 $\big/$ \textbf{0.130}&
63.991 $\big/$ \textbf{14.524}   &
115.732 $\big/$ \textbf{80.857}   &
\textbf{0.079} $\big/$ 0.189 \\

UAP$^{*}$ &
17.765 $\big/$ \textbf{16.363}&
19.775 $\big/$ \textbf{17.188}& 
16.997 $\big/$ \textbf{16.847}& 
\textbf{23.109} $\big/$ 27.832& 
10.583 $\big/$ \textbf{8.131}& 
14.833 $\big/$ \textbf{10.922}&  
20.813 $\big/$ \textbf{19.434}& 
5.795 $\big/$ \textbf{5.592}\\ 

Kor.$^{*}$ & 
18.564 $\big/$ \textbf{17.667}& 
\textbf{12.617} $\big/$ 12.707 & 
19.500 $\big/$ \textbf{17.865}& 
18.423 $\big/$ \textbf{17.395}& 
13.698 $\big/$ \textbf{10.107}& 
6.514 $\big/$ \textbf{5.298}& 
14.807 $\big/$ \textbf{12.407}& 
7.759 $\big/$ \textbf{6.680}\\ 
\bottomrule
\end{tabular}
}
\label{tab:RMSE}
\end{table*}

\begin{table*}[!t]
\caption{The \textbf{SROCC$\uparrow$} metric of NR-IQA models against attacks (``baseline $\big /$ baseline+NT''). \textbf{Bold} denotes better value in a cell.}
\centering
\renewcommand\arraystretch{1.3}
\resizebox{0.94\textwidth}{!}{
\begin{tabular}{lcccc|cccc}
\toprule
  &
  \multicolumn{4}{c}{MOS \& Predicted Score After Attack} &
  \multicolumn{4}{c}{Score Before Attack   \& Score After Attack} \\ \cmidrule(lr){2-5} \cmidrule(lr){6-9}
  & {\makecell[c]{HyperIQA \\ base / +NT }}            & {\makecell[c]{DBCNN \\ base / +NT }}    & {\makecell[c]{LinearityIQA \\ base / +NT }}    & {\makecell[c]{MANIQA \\ base / +NT }}  & {\makecell[c]{HyperIQA \\ base / +NT }}            & {\makecell[c]{DBCNN \\ base / +NT }}    & {\makecell[c]{LinearityIQA \\ base / +NT }}    & {\makecell[c]{MANIQA \\ base / +NT }}\Bstrut\\ 
  \hline
FGSM &
  0.021 $\big/$ \textbf{0.810}&
  -0.318 $\big/$ \textbf{0.200} &
  -0.375 $\big/$ \textbf{-0.347} &
  0.417 $\big/$ \textbf{0.772} &
  0.043 $\big/$ \textbf{0.941}&
  -0.333 $\big/$ \textbf{0.227} & 
  -0.429 $\big/$ \textbf{-0.426} &
  0.428 $\big/$ \textbf{0.878} \\
Perceptual &
  0.815 $\big/$ \textbf{0.858}& 
  -0.127 $\big/$ \textbf{0.643} & 
  0.477 $\big/$ \textbf{0.567} & 
  \textbf{0.876} $\big/$ 0.871 & 
  0.938 $\big/$ \textbf{1.000}& 
  -0.160 $\big/$ \textbf{0.773} & 
  0.542 $\big/$ \textbf{0.685} & 
  \textbf{1.000} $\big/$ \textbf{1.000} \\
UAP$^{*}$ &
  0.736 $\big/$ \textbf{0.822}& 
  0.705 $\big/$ \textbf{0.760} & 
  0.715 $\big/$ \textbf{0.739} &
  0.773 $\big/$ \textbf{0.839} & 
  0.825 $\big/$ \textbf{0.941}& 
  0.836 $\big/$ \textbf{0.887} & 
  0.836 $\big/$ \textbf{0.869} & 
  0.923 $\big/$ \textbf{0.976} \\ 
Kor.$^{*}$ &
  \textbf{0.808} $\big/$ 0.802& 
  \textbf{0.863} $\big/$ 0.856 &
  \textbf{0.775} $\big/$ \textbf{0.775} &
   0.828  $\big/$ \textbf{0.847}& 
  0.892 $\big/$ \textbf{0.922}& 
  0.978 $\big/$ \textbf{0.983}&
  \textbf{0.936} $\big/$ \textbf{0.936} &
   0.942 $\big/$ \textbf{0.969}\\ 
\bottomrule
\end{tabular}
}
\label{tab:SROCC}
\end{table*}

\begin{observation}
    The NT strategy significantly improves the robustness of NR-IQA models in most cases, where the robustness is in terms of RMSE, SROCC, PLCC, KROCC or $R$ robustness.
\end{observation}

Due to space constraints, we only present the robustness results for RMSE (Table~\ref{tab:RMSE}) and SROCC (Table ~\ref{tab:SROCC}) in this subsection. 
Comprehensive results for PLCC, KROCC and $R$ robustness can be found in the supplementary material. 
In both tables, columns 2-5 display the IQA metric calculated between MOS values of unattacked images and predicted scores on adversarial examples, while columns 6-9 showcase the metric calculated between predicted scores on unattacked images and scores on adversarial examples.

As shown in Table~\ref{tab:RMSE}, when RMSE is computed between predicted scores before and after attacks, NR-IQA models trained with the NT strategy exhibit smaller score changes under nearly all attack scenarios compared to baseline models.
These results confirm the correctness of our theoretical analysis in Sec.~\ref{sec:method_why}.
The robustness improvement is especially significant when models are attacked by FGSM.
The only exception is MANIQA when attacked by the Perceptual Attack.
In this case, the RMSE of MANIQA is smaller than that of the NT-trained model, where the difference is only 0.11.
We think this phenomenon can be attributed to inherent biases among test images.
Furthermore, MANIQA and its NT-enhanced version exhibit significant robustness against the Perceptual Attack, as indicated by an SROCC value of 1 between predicted scores before and after the attack.
This signifies that the Perceptual Attack has minimal impact on MANIQA and MANIQA+NT, resulting in a reasonably small difference in RMSE between the two models.

When considering RMSE results measured between MOS values and predicted scores after attacks, the robustness of NT-trained models is also improved.
For example, the RMSE value of DBCNN under the FGSM attack is about 36.758, whereas that of DBCNN+NT is just 24.711.
There are only 3 out of 16 cases where baseline+NT models perform worse than baseline models.
Such occurrences are expected because the NT strategy does not leverage MOS information but relies on the original predicted scores.

Results shown in Table~\ref{tab:SROCC} demonstrate that the NT strategy can also enhance the robustness in terms of SROCC, although we are not clear about the theoretical connection between the NT strategy and SROCC.
The improvement is particularly pronounced in white-box scenarios.
Taking the HyperIQA model as an example, the robustness of the baseline model measured by SROCC is notably deficient under the FGSM attack.
The SROCC value between MOS values and predicted scores is a mere 0.021, while the SROCC value between predicted scores before and after the FGSM attack is only 0.043.
However, with the inclusion of the NT strategy, there is a significant enhancement in SROCC. 
The SROCC value between MOS values and predicted scores increases to 0.810, while the SROCC between scores before and after the FGSM attack rises to 0.941. 
This exemplifies the effectiveness of the NT strategy in boosting the SROCC robustness of NR-IQA models.

\begin{observation}
    In IQA tasks, the robustness in terms of distinct metrics is not completely the same. %
\end{observation}

In Table~\ref{tab:RMSE} and Table~\ref{tab:SROCC}, it is evident that a model showing robustness in terms of RMSE when subjected to an attack method may not necessarily exhibit robustness in SROCC.
Take the left part of the two tables as an example, for baseline models, we can see that HyperIQA achieves much better robustness in terms of RMSE than MANIQA against the UAP attack (17.765 vs. 23.109), but it performs worse in SROCC (0.736 vs. 0.773).
Similar phenomena also occur with NT-trained models where LinearityIQA+NT shows better RMSE robustness but worse SROCC robustness than DBCNN+NT against the UAP attack.
How to make a trade off between the adversarial robustness from different perspectives brings challenges in IQA tasks, potentially opening up new avenues for further exploration and research.

\begin{observation}
    The NT strategy exhibits a more effective defense against white-box attacks compared to black-box attacks.
\end{observation}

Table~\ref{tab:RMSE} and Table~\ref{tab:SROCC} demonstrate that the improvement in RMSE / SROCC from the baseline to baseline+NT models is generally greater under white-box attacks than under black-box attacks.
This trend is more clear from Table S7 in the supplementary material by comparing the averaged metrics of improvement.

This happens because the attack capability of existing black-box attacks is generally weaker than that of white-box attacks on NR-IQA models. Hence, the baseline models exhibit better robustness against black-box attacks compared to white-box attacks, making the robustness improvement brought by NT less evident in the black-box scenario. For instance, when attacked by black-box methods, the SROCC values between predicted scores before and after attacks for all baseline models exceed 0.8, while these values for most baseline models under white-box attacks are below 0.6.
This observation highlights the importance of exploring effective black-box attacks on IQA models.

\subsection{Norm Reduction}
\label{sec:ex_norm}

To validate the effectiveness of the NT strategy in reducing the norm of the gradient, as well as the accuracy of Eq.~\eqref{eq:approx_l1} in approximating the $\ell_1$ norm, we generate distribution plots of $\Vert \nabla_x f(x) \Vert_1$.
Here, $x$ represents samples from the test set.

Figure~\ref{fig:norm_dis} compares the norm distribution between baseline and baseline+NT models.
We can see that the gradient norms of the baseline+NT models are all shifted towards the left compared to the baseline models.
This indicates that models trained with the NT strategy exhibit a smaller $\ell_1$ gradient norm concerning the input image compared to the baseline models.
These results confirm that Eq.~\eqref{eq:approx_l1} serves as a reliable approximation of the gradient $\ell_1$ norm.

\begin{figure}[!t]
    \centering
    \includegraphics[width=\linewidth]{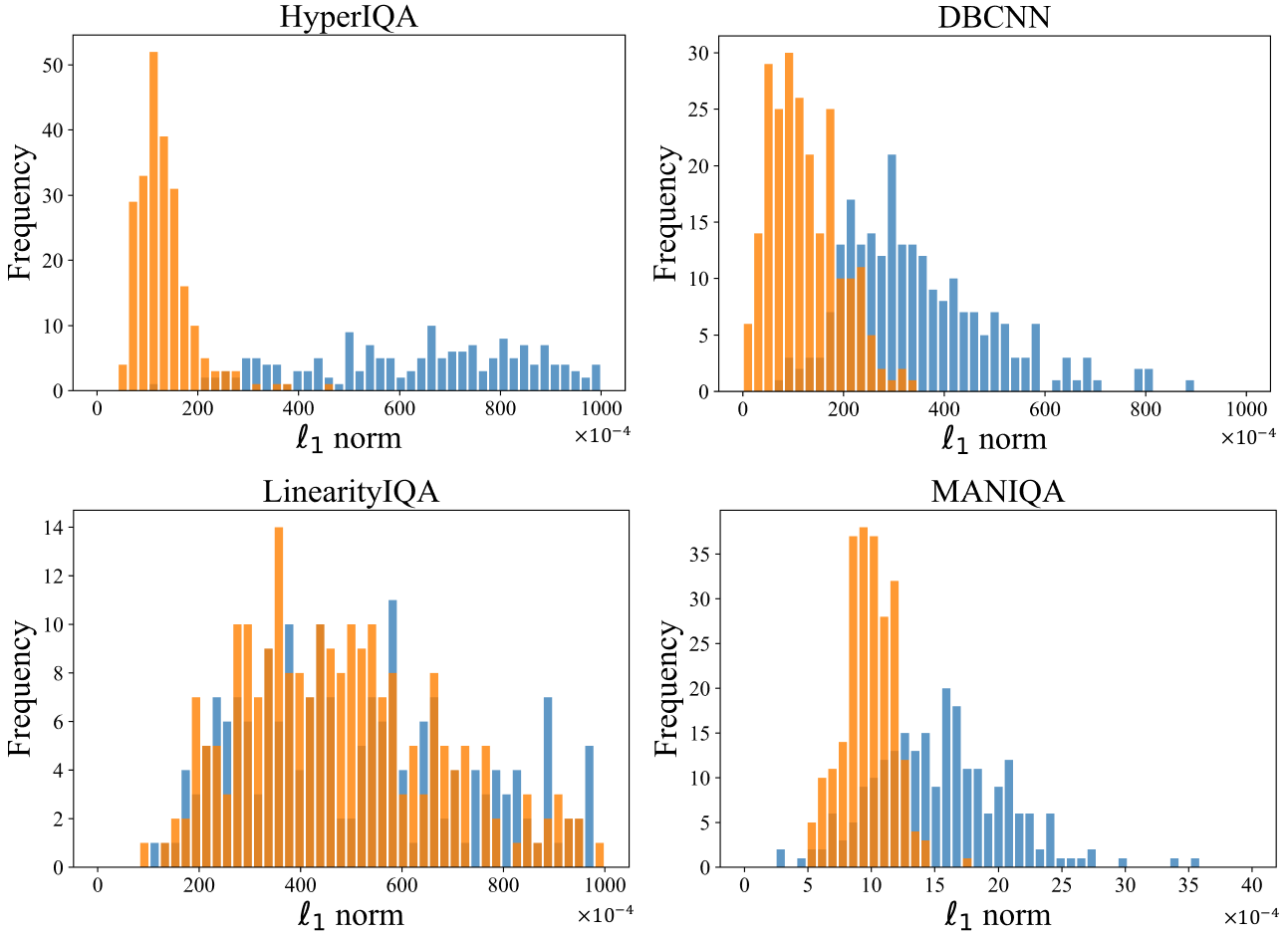}
    \caption{The comparison of $\ell_1$ norm distribution of gradient between baseline models (blue) and baseline+NT models (orange).}
    \label{fig:norm_dis}
\end{figure}

\begin{figure}[!thb]  
    \centering
    \subfloat{
    \includegraphics[width=0.475\linewidth]{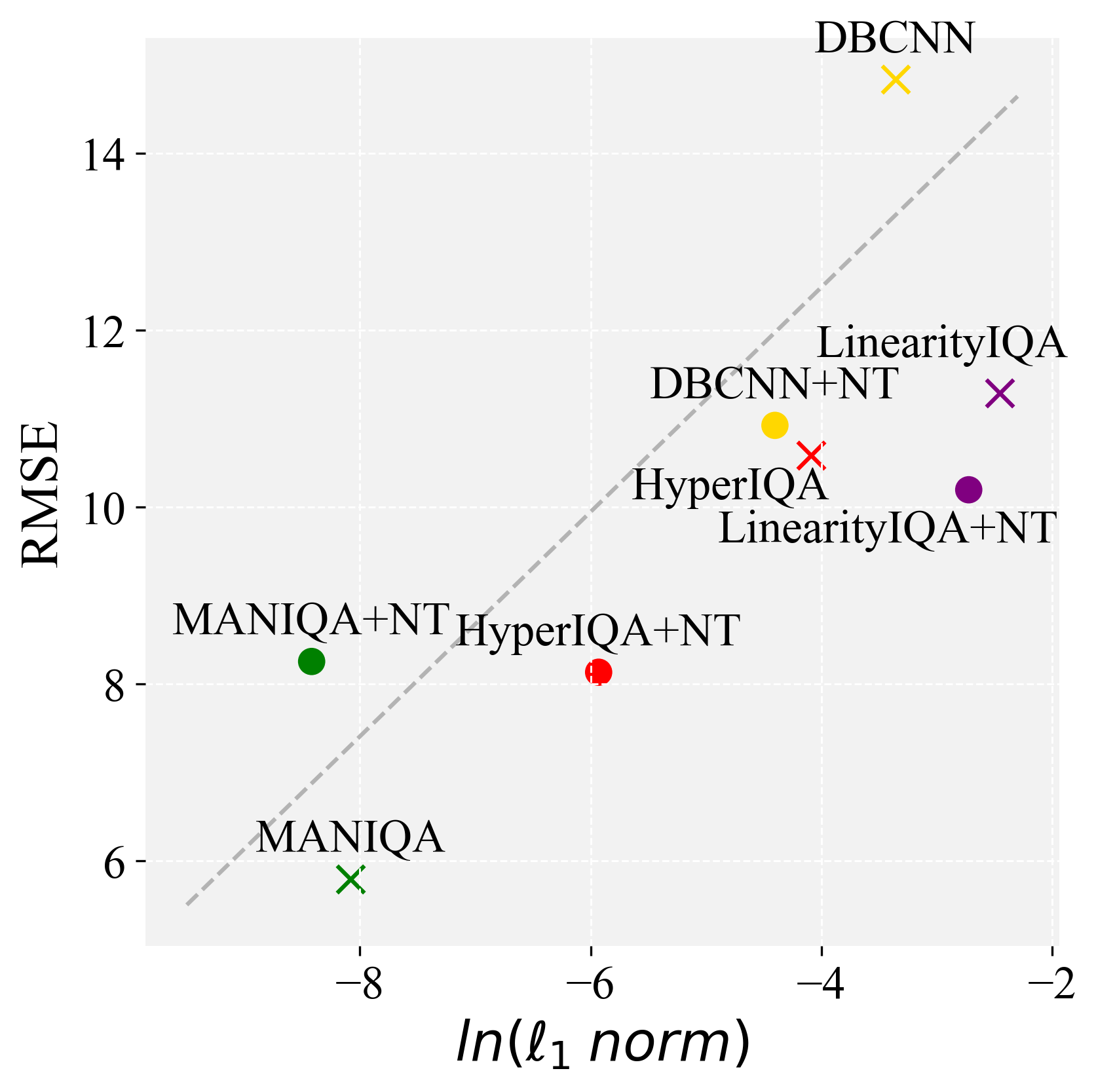}}
    \subfloat{
    \includegraphics[width=0.495\linewidth]{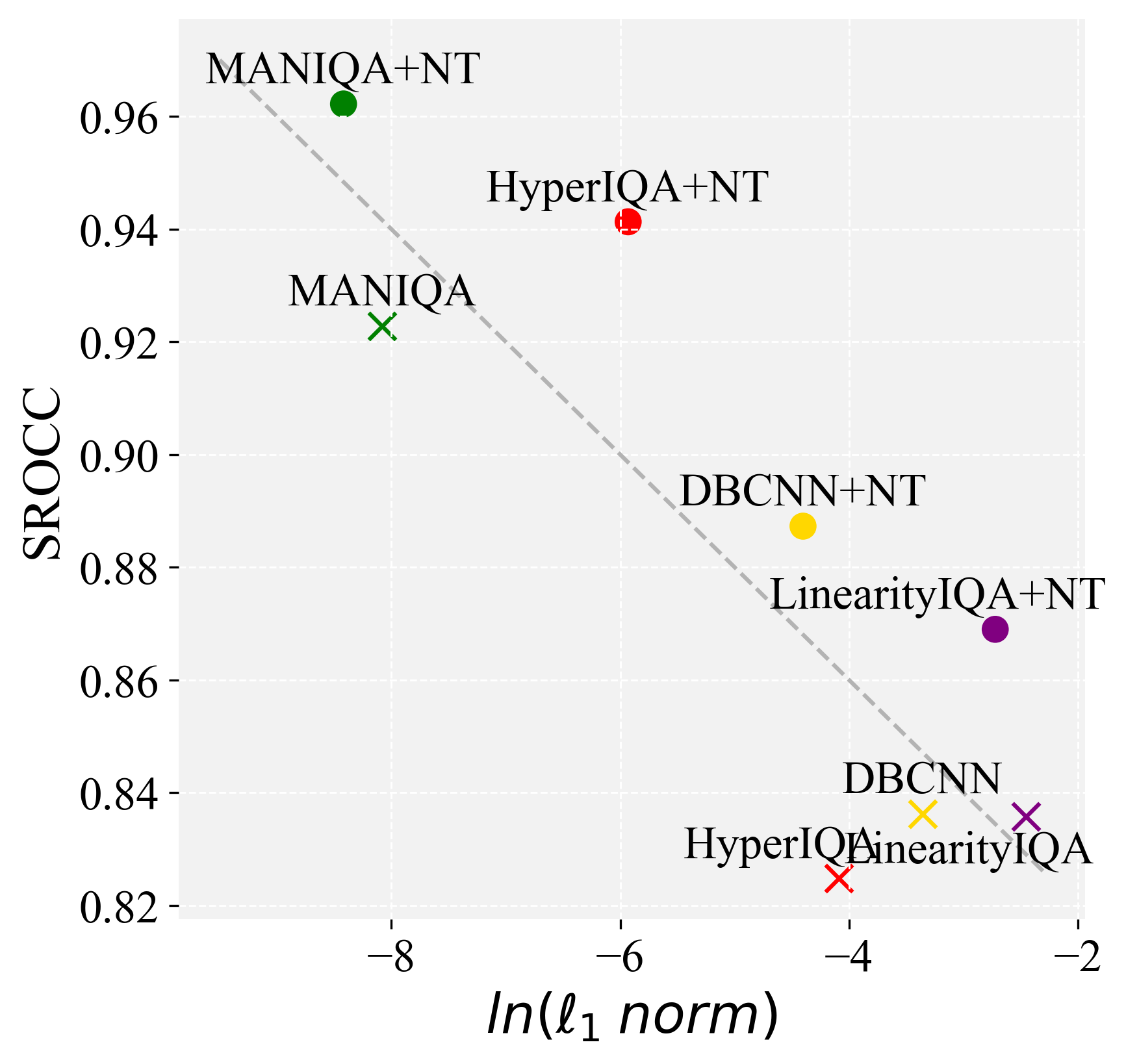}}
    \caption{The relationship between the gradient norm and the robustness in terms of RMSE (left) and SROCC (right). The horizontal axis represents the logarithm of the average $\Vert \nabla_x f(x)\Vert_1$ value across all test images. All metrics are calculated between predicted scores before and after the UAP attack.}
    \label{fig:norm_robustness}
\end{figure}

To further demonstrate that a smaller $\Vert \nabla_x f\Vert_1$ enhances the robustness of an NR-IQA model against adversarial attacks, we draw a scatter plot to show the relationship between the adversarial robustness and the gradient norm.
In Figure~\ref{fig:norm_robustness}, the horizontal axis represents the logarithm of the average $\Vert \nabla_x f(x)\Vert_1$ value across all test images.
Points in the left part of Figure~\ref{fig:norm_robustness} generally follow a diagonal distribution from bottom left to top right, indicating that models with smaller gradient norms tend to exhibit better robustness in terms of RMSE.
Moreover, points in the right part of Figure~\ref{fig:norm_robustness} are generally distributed from the top left to the bottom right.
This reflects that models with smaller gradient norms tend to exhibit better robustness in terms of SROCC.

\subsection{Attack Intensity and Robustness}
\label{sec:ex_intensity}
To evaluate the robustness of the baseline model and the baseline+NT model under different attack intensities, we adjust the strength of the iterative FGSM attack (illustrated in the supplementary material) with different iterations and the $\ell_\infty$ norm $\epsilon$ of perturbations. 
Generally, a higher number of iterations and larger $\epsilon$ values correspond to more potent attacks.
The attack intensity is quantified using SSIM, with smaller SSIM values signifying greater attack intensity.

Figure~\ref{intensity} presents the performance of HyperIQA and its NT-trained versions under attacks with varying intensities. 
As the attack intensities increase, the RMSE and SROCC values for HyperIQA and its NT version tend to get worse in general. This reflects that stronger attacks lead to decreased performance for both normally-trained and NT-trained models in most cases.
Meanwhile, HyperIQA+NT model consistently keeps lower RMSE values than HyperIQA at the same attack intensity, regardless of the intensity levels.
This demonstrates the effectiveness of the NT strategy against attacks with varying intensities. 
\begin{figure}[!t]
    \centering
    \subfloat{
    \includegraphics[width=0.495\linewidth]{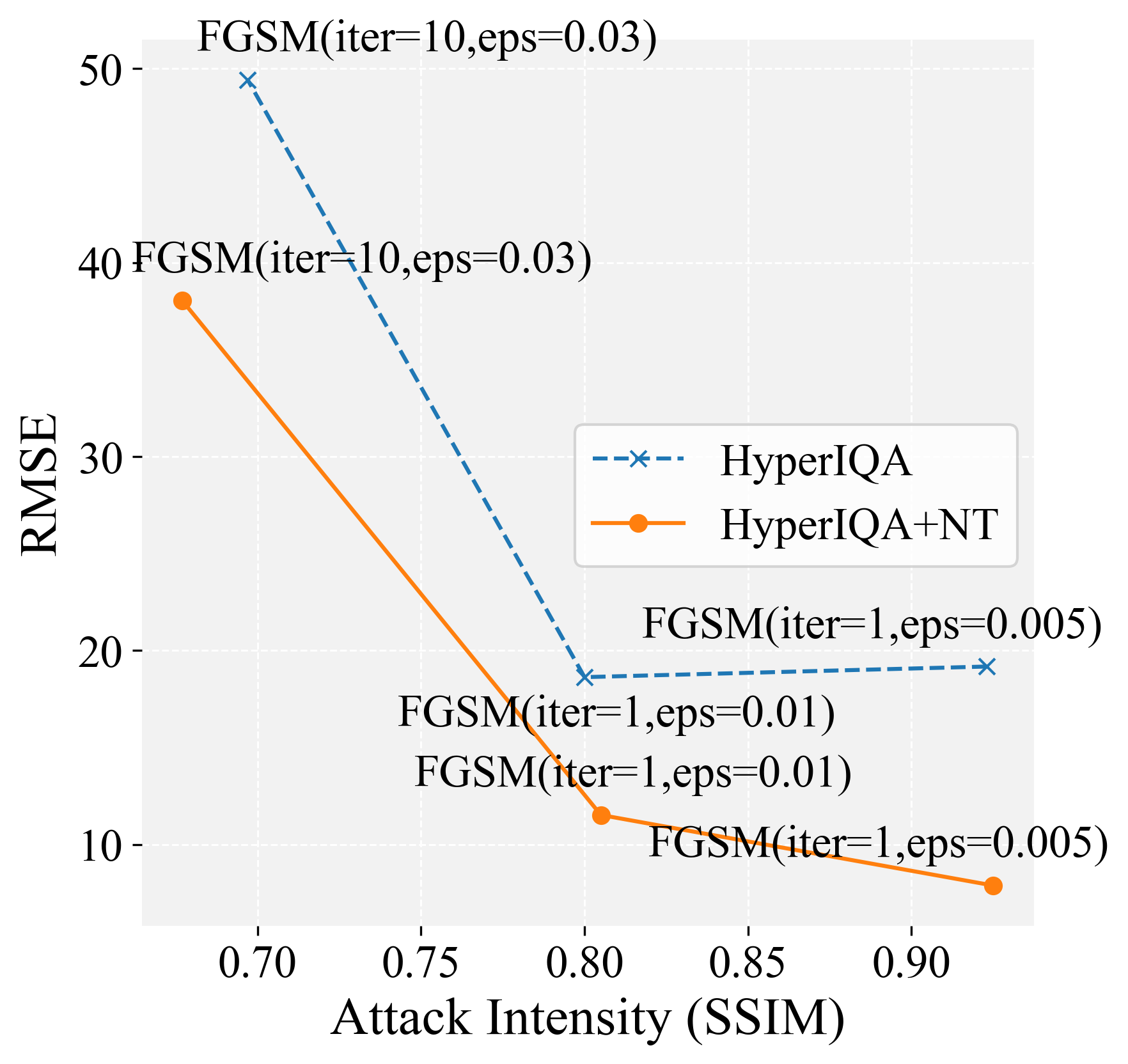}}
    \subfloat{
    \includegraphics[width=0.495\linewidth]{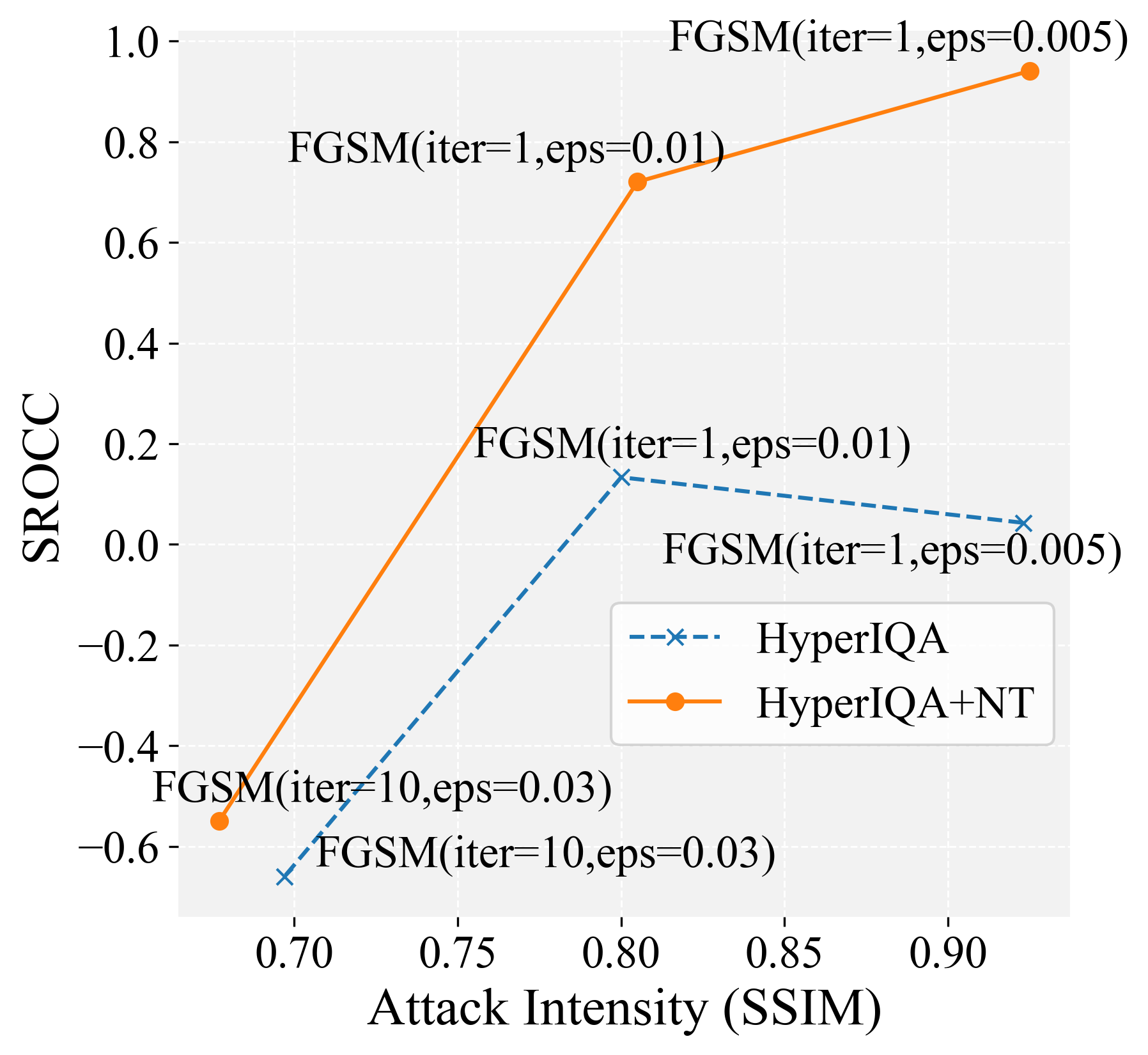}}
    \caption{RMSE (left) and SROCC (right) values of HyperIQA and HyperIQA+NT models under different attack intensities. RMSE and SROCC are calculated between predicted scores before and after the FGSM attack.} 
    \label{intensity}
\end{figure}

\subsection{Ablation Study}\label{sec:ex_ablation}
We conduct additional experiments to test the impact of hyperparameters in Eq.~\eqref{eq:training_loss} for the NT strategy: the weight $\lambda$ of the gradient norm and the step size $h$ in the finite difference.
Due to the space limit, we present partial results and full results are shown in the supplementary material.

In Figure~\ref{fig:lambda}, we fix $h=0.01$ and vary $\lambda$ in Eq.~\eqref{eq:training_loss} in the range from $0$ to $0.003$.
Our analysis focuses on two aspects of an NR-IQA model: its performance on unattacked images and its robustness against attacks. For the former, we utilize SROCC on unattacked images across MOS values and predicted scores, and for the latter, we employ the RMSE between predicted scores before and after the FGSM attack. 
As $\lambda$ increases, SROCC values on unattacked images tend to decrease on all baseline+NT models,  while the RMSE values under
the FGSM attack tends to decrease consistently. 
This implies that increasing $\lambda$ enhances the robustness of NR-IQA models but leads to a performance decline on unattacked images.

\begin{figure}[!th]
    \centering
    \includegraphics[width=0.9\linewidth]{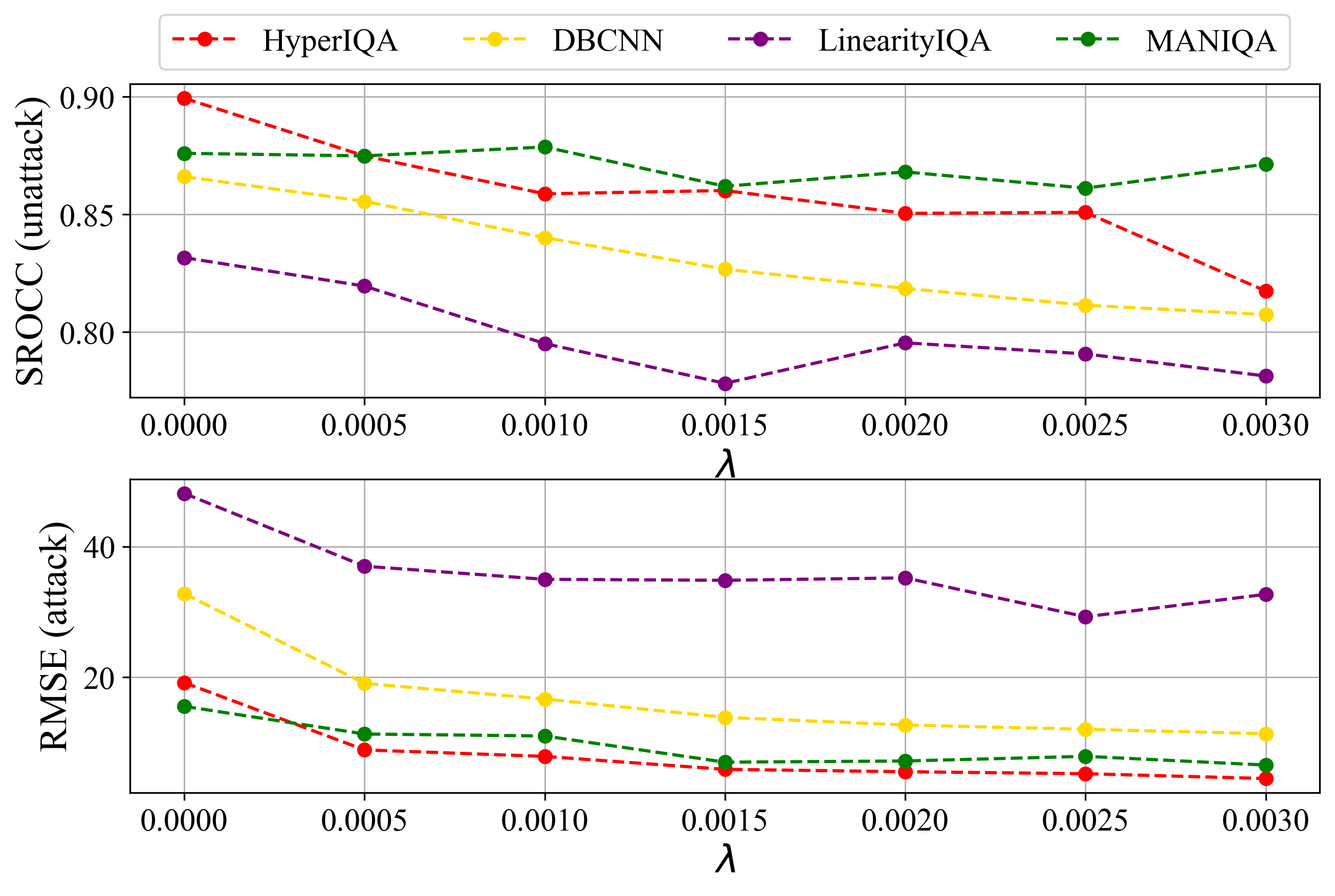} %
    \caption{The impact of $\lambda$ to SROCC on unattacked images and RMSE on FGSM attacked images.}
    \label{fig:lambda}
\end{figure}

To explore the effect of the step size $h$ in Eq.~\eqref{eq:training_loss} on the performance of IQA models, we fix $\lambda=0.0005$ and vary $h$ in $\{0.001, 0.01, 0.1, 1\}$ for DBCNN.
In Table~\ref{tab:h}, we present the SROCC and RMSE values across MOS values and predicted scores for unattacked images, and SROCC and RMSE values between predicted scores before and after the FGSM attack for adversarial examples.
In theory, a large $h$ cannot sufficiently represent the neighborhood of $x$, so the approximation of the $\ell_1$ norm is inaccurate.
The experimental results also confirm this point where the robustness of the model is worse when $h=0.1$ and $1$.
Conversely, an exceedingly small $h$, such as $h=0.001$, achieves effective defense performance but leads to a significant performance decline on unattacked images.

\begin{table}[!t]
\centering
\caption{The comparison of different $h$ of DBCNN+NT model with unattacked images and FGSM attacked images.}
\label{tab:h}
\resizebox{0.9\linewidth}{!}{
\begin{tabular}{cccccc}
\toprule
    &  & \multicolumn{4}{c}{$h$} \\ \cmidrule(lr){3-6} 
    &  & 
    {\makecell[c]{0.001 }} & {\makecell[c]{0.01 }} & {\makecell[c]{0.1 }} & {\makecell[c]{ 1 }}\\ \midrule
\multirow{2}{*}{Unattacked} & SROCC$\uparrow$  & 0.788  & 0.856  & 0.846   & 0.844   \\
                            & RMSE$\downarrow$  & 16.099 & 14.138 & 12.417  & 14.809  \\ \midrule
\multirow{2}{*}{Attacked}   & SROCC$\uparrow$ & 0.577  & 0.200  & -0.3832 & -0.4406 \\
                            & RMSE$\downarrow$  & 7.356  & 19.065 & 28.785  & 18.767  \\ \bottomrule
\end{tabular} %
}
\end{table}

\vspace{-0.1in}
\section{Conclusion}

To the best of our knowledge, this is the first work designing IQA-specific defense methods against adversarial attacks.
Our work offers a rigorous theoretical proof that the score changes of NR-IQA models are related to the $\ell_1$ norm of the gradient when the perturbation is small.
Furthermore, models trained with the proposed NT strategy exhibit significant improvement in adversarial robustness against both white-box and black-box attacks.

{\noindent \textbf{Limitations and future work.}}
In this study, our primary theoretical analysis is on enhancing the prediction accuracy of NR-IQA models, \ie, reducing the changes in predicted scores when NR-IQA models are exposed to attacks. Nevertheless, an interesting and valuable avenue for future research is the development of NR-IQA models that demonstrate robustness in terms of prediction monotonicity  like SROCC. Besides, we intend to explore our method applied to FR-IQA models in future work.

{\noindent \textbf{Acknowledgements.}} This work is partially supported by Sino-German Center (M 0187) and the NSFC under contract 62088102. 
Thank Ruohua Shi for supports during rebuttal.
Thank Zhaofei Yu and Yajing Zheng for their valuable suggestions on the writing. 
Thank High-Performance Computing Platform of Peking University for providing computational resources. 

{\small
\bibliographystyle{ieeenat_fullname}
\bibliography{main}
}

\newpage

\renewcommand{\thetable}{S\arabic{table}}
\renewcommand{\thefigure}{S\arabic{figure}}
\renewcommand{\thesection}{S\arabic{section}}
\renewcommand{\theequation}{S\arabic{equation}}

\section*{Supplementary Material}

In the supplementary material, we offer the formulation of NR-IQA metrics (Sec.~\red{3.1}), detailed proofs of the finite difference in Eq.~(6) (Sec.~\red{4.2}), additional implementation details (Sec.~\red{5.1}), further robustness analysis (Sec.~\red{5.2}), and supplementary ablation study results (Sec.~\red{5.5}). Additionally, we present more visualization results.

\section{Formulations of RMSE, SROCC, KROCC, PLCC and $R$ Robustness}
In this section, we will introduce IQA-specific metrics RMSE, SROCC, KROCC, PLCC, and $R$ robustness mentioned in Sec~\textcolor{red}{3.2}.

\textbf{RMSE} measures the difference between MOS values and predicted scores, which is represented as
\begin{equation}
    \text{RMSE} = \sqrt{\frac{1}{N}
    \sum_{i=1}^N (y_i-f_i)^2}.
\end{equation}
In this equation, $N$ is the number of images. $y_i$ and $f_i$ represent the MOS and predicted score of the $i^{th}$ image, respectively. The smaller the RMSE value is, the smaller the differences between the two groups of scores.

\textbf{SROCC} measures the correlation between MOS values and predicted scores to what extent the correlation can be described by a monotone function. The specific formulation is as follows:
\begin{equation}
    \text{SROCC} = 1 - \frac{6 \sum_{i=1}^N d_i^2}{ N(N^2-1)},
\end{equation}
where $d_i$ denotes the difference between orders of the $i^{th}$ image in subjective and objective quality scores. The closer the SROCC value is to 1, the more consistent the ordering is between two groups of scores.

\textbf{KROCC} measures the degree of concordance in the ranking of MOS values and predicted scores. The formulation is:
\begin{equation}
    \text{KROCC} = \frac{2 (N_{\text{con}} - N_{\text{dis}})}{ N(N-1)}.
\end{equation}
In this equation, $N_{\text{con}}$ and $N_{\text{dis}}$ represent the number of image pairs in the test dataset with consistent and inconsistent ranking of subjective and objective quality scores, respectively. The closer the KROCC value is to 1, the more consistent the ordering is between two groups of scores.   

\textbf{PLCC} measures the linear correlation between MOS values and predicted scores, which is formulated as
\begin{equation}
\begin{split}
    \text{PLCC} &= \frac{
    \sum_{i=1}^N (y_i-\Bar{y})(f_i - \Bar{f})
    }{
    \sum_{i=1}^N (y_i-\Bar{y})^2(f_i - \Bar{f})^2
    }, \\
    \Bar{y} & = \frac{1}{N} \sum_{i=1}^N y_i,  \Bar{f} = \frac{1}{N} \sum_{i=1}^N f_i.
\end{split}
\end{equation}
The closer the PLCC value is to 1, the higher the positive correlation between the two groups of scores.

\textbf{$R$ robustness} was recently proposed by~\citet{2022_NIPS_Zhang_PAttack}.
It takes the maximum allowable change in quality prediction into consideration:
\begin{equation}
    R=\frac{1}{N}\sum_{i=1}^{N}\log\left(\frac{\max\{\beta_1-f(x_i),f(x_i)-\beta_2\}}{|f(x_i)-f(x'_i)|}\right),
\end{equation}
where $N$ is the number of images, $x_i$ is the $i_\text{th}$ image to be attack, $x'_i$ is the attacked version of $x_i$. $f(\cdot)$ is the IQA model for quality prediction. $\beta_1$ and $\beta_2$ are the maximum MOS and minimum MOS among all MOS values. A larger $R$ value means better robustness.

\section{Proof of Eq. (6)}

Eq. (6) in the main context illustrates how to approximate the $\ell_1$ norm of $\nabla_x f(x)$ by the finite difference, which is expressed as
\begin{equation}
    \Vert \nabla_x f(x) \Vert_1 \approx 
    \left| \frac{
        f(x+h\cdot d) - f(x)
        }{
        h
        } \right|.
\nonumber
\end{equation}
In this formula, $h$ is a small step size and $d=\text{sign}(\nabla_x f(x))$.
We provide a proof of this approximation in this section.

\begin{proof}
     We expand $f(x+h\cdot d)$ at point $x$ by the first order Taylor estimation, \ie,
    \begin{equation}
        f(x+h\cdot d) \approx f(x) + h\cdot \nabla_x f(x)^T d.
    \end{equation}
    Since $d=\text{sign}(\nabla_x f(x))$, we have
    \begin{equation}
         \nabla_x f(x) d = \Vert \nabla_x f(x) \Vert_1.
    \end{equation}
    Therefore,
    \begin{equation}
        f(x+h\cdot d) \approx f(x) + h\Vert \nabla_x f(x) \Vert_1,
    \end{equation}
    and $\Vert \nabla_x f(x) \Vert_1$ can be approximated by
    \begin{equation}
        \left| \frac{
        f(x+h\cdot d) - f(x)
        }{
        h
        } \right|.
    \end{equation}
\end{proof}

\section{Experimental Settings}
\label{supp:train_set}
In Sec.~\textcolor{red}{5.1} in the main manuscript, part of the experimental settings are reported. In this section, we report the experimental environment and detailed experimental settings in our experiments.

\subsection{Experimental Environment}

We conducted all the training, test, and attack on an NVIDIA GeForce RTX 2080 GPU with 11GB of memory. 

\subsection{Training Settings}
For the four NR-IQA models considered in our study, namely, HyperIQA~\cite{2020_CVPR_hyperIQA}, DBCNN~\cite{2020_TCSVT_DBCNN}, LinearityIQA~\cite{2020_MM_LinearityIQA}, and MANIQA~\cite{2022_CVPRw_MANIQA}, we used publicly available code provided by their respective authors to train these models on the same training dataset. 

Due to the memory requirements associated with approximating the $\ell_1$ norm, the batch size used for training the NR-IQA models had to be adjusted to prevent memory overflow. Other training settings are shown in Table~\ref{tab:batch_size}. To ensure consistency and fairness in our comparisons, the same setting is utilized when training both the baseline and NT versions of each NR-IQA model.

\begin{table}[htbp]
\caption{Detailed training settings for NR-IQA models and their NT versions. ``Patches per Image'' is marked as ``-'' if the input of the model is the whole image}
\label{tab:batch_size}
\small
\centering
\resizebox{\linewidth}{!}{
\begin{threeparttable}
\begin{tabular}{lcccccc}
\toprule
Model & Architecture & \makecell[c]{Input\\Size} & \makecell[c]{Patches\\per Image} & \makecell[c]{Batch\\Size} & \makecell[c]{Training\\Epochs} \\ 
\midrule
  \makecell[l]{HyperIQA /\\ HyperIQA+NT} & 
 ResNet50   & 224$\times$224 &  25 & 16 & 16 \\ 
  \specialrule{0em}{.4ex}{.65ex}
  \makecell[l]{DBCNN /\\ DBCNN+NT} & 
 \makecell[c]{VGG and\\Its Variant}  & 500$\times$500 & - &6 & 50 \\ 
  \specialrule{0em}{.4ex}{.65ex}
  \makecell[l]{LinearityIQA /\\ LinearityIQA+NT}  & 
  ResNet34 & 498$\times$664 & - & 4 & 30 \\ 
  \specialrule{0em}{.4ex}{.65ex}
  \makecell[l]{MANIQA /\\ MANIQA+NT} & ViT-B/8
   & 224$\times$224 & 20 & 1 & 30 \\ 
\bottomrule
\end{tabular}
\end{threeparttable}
}\end{table}
\subsection{Normalization of MOS}
In our selected 4 NR-IQA models, MANIQA~\cite{2022_CVPRw_MANIQA} is a special NR-IQA model in which MOS is scaled to the range of $[0,1]$, and it leads to the predicted score in the range of $[0,1]$. Furthermore, the different scales of MOS in different NR-IQA models result in a difference in the RMSE metric.

For a fair comparison across different NR-IQA models, we normalize the MOS into the range $[0,100]$. The normalization formula is depressed as follows:
\begin{equation}
    {\text{MOS}}_n = \frac{\text{MOS} - S_\text{min}}{S_\text{max}  - S_\text{min}} \times 100. %
\end{equation}
In this formula, $S_\text{min}$ and $S_\text{max}$ represent the minimal and maximal MOS of the training data, respectively.
In this paper, $S_\text{min}=3.42$ and $S_\text{max}=92.43$. %

\section{(I-)FGSM for NR-IQA Tasks}
\label{sec:fgsm}
We mention the FGSM attack in Sec.~\textcolor{red}{5.1} in the main manuscript. We will introduce the details of the setting of the FGSM attack in this section.
FGSM~\cite{2015_ICLR_Goodfellow_FGSM} is first proposed for classification tasks, which is concise and efficient in attacking classification models. In our paper, we perceive FGSM as a white-box attack for NR-IQA models with a redesigned loss function. We will first introduce the FGSM attack in classification tasks and then the FGSM attack adapted to NR-IQA models below.

In the context of classification, FGSM is a straightforward non-iterative attack method, which is expressed as follows:
\begin{align}
\label{eq:FGSM}
    x_{\text{adv}} = x + \epsilon~\text{sign} (\nabla_{x} \mathcal{L} (f(x),y)).
\end{align}
In this equation, $x_{\text{adv}}$ represents the adversarial example, $x$ is the original image, $\epsilon$ denotes the $\ell_\infty$ norm bound of perturbations, $\mathcal{L}$ is the loss function, $f(\cdot)$ signifies the neural network function, and $y$ represents the true label of $x$. A common use of $\mathcal{L}$ is cross-entropy loss.

I-FGSM is an iterative extension of FGSM, which is described as follows:
\begin{equation}
    x_{\text{adv}}^k =  \Pi_\epsilon \left\{
    x_{\text{adv}}^{k-1} + \alpha~\text{sign} (\nabla_{x} \mathcal{L} (f(x),y))
    \right\},
\end{equation}
where $k$ is the current iteration step and $\alpha$ is the step size, the total number of iteration steps is $K$. 
The operator $\Pi_\epsilon$ projects the adversarial examples onto the space of the $\epsilon$ neighborhood in the $\ell_\infty$-ball around $x$.

In the NR-IQA task, we take $y$ as the predicted score of the clean image $x$.
In this paper, we choose the optimization object according to the predicted score of the image and define the loss function $\mathcal{L}$ as follows:
\begin{equation}
    \mathcal{L}(f(x),y) \triangleq \mathcal{L}_{\text{mid}} = \left\{
    \begin{aligned}
        f(x), \quad y \leqslant 50, \\
        -f(x), \quad y > 50,
    \end{aligned}
    \right.
\label{eq:supp_fgsm_used}
\end{equation}
where $f(x)$ represents the predicted score of the attacked image.
The object is to maximize the predicted score for a low-quality image, thereby misleading the IQA model into assigning a high score to the adversarial example.
Conversely, for a high-quality image, the goal is to minimize the predicted score to generate effective adversarial examples.

There is an interesting observation emerged from the experiment. 
We find that the choice of the loss function $\mathcal{L}$ has a significant impact on the efficacy of the FGSM attack. 
Specifically, we also try the mean absolute error loss:
\begin{equation}
    \mathcal{L}(f(x),y) \triangleq \mathcal{L}_{\text{mae}} = | f(x)-y |,
\label{eq:supp_fgsm_abs}
\end{equation}
and the mean squared error loss:
\begin{equation}
    \mathcal{L}(f(x),y) \triangleq \mathcal{L}_{\text{mse}} = (f(x)-y)^2.
\label{eq:supp_fgsm_mse}
\end{equation}
Taking the DBCNN as an example, we report the RMSE, SROCC, PLCC, and KROCC after the FGSM attack with different loss functions in Table~\ref{tab:IFGSM}\footnote{As for attack methods, larger RMSE and smaller SROCC, KROCC, PLCC represents stronger attack ability.} where $\epsilon=0.005$.
It is obvious that the effect of the FGSM attack is notably diminished when the loss function is $\mathcal{L}_{\text{mae}}$ or $\mathcal{L}_{\text{mse}}$.
Especially when the mean absolute loss $\mathcal{L}_{\text{mae}}$ is used, the changes of RMSE, SROCC, PLCC, and KROCC for all models are very minimal.

Investigating the relationship between the loss function and the ability of attacks is an interesting domain of research.

\begin{table}[htbp]
\caption{The attack ability of the FGSM attack with different loss functions. \textbf{Bold} denotes better value in a column}
\centering
\renewcommand\arraystretch{1.2}
\resizebox{0.8\linewidth}{!}{
\begin{tabular}{lcccc}
\toprule
 \multicolumn{5}{c}{MOS \& Predicted Score After Attack} \\ \midrule
       & RMSE$\uparrow$          & SROCC$\downarrow$    & PLCC$\downarrow$  & KROCC$\downarrow$  \\ \midrule
$\mathcal{L}_{\text{mae}}$  &
  10.0734 & 0.8994 & 0.8844 & 0.7177 \\
$\mathcal{L}_{\text{mse}}$ & 
24.354  & 0.2795 & 0.2092 & 0.2096 \\ 
$\mathcal{L}_{\text{mid}}$   & 
\textbf{ 36.758} & \textbf{-0.318} & \textbf{-0.383} & \textbf{-0.146}\\ \midrule
\multicolumn{5}{c}{Predicted Scores Before \& After Attack} \\ \midrule
       & RMSE$\uparrow$         & SROCC$\downarrow$    & PLCC$\downarrow$   & KROCC$\downarrow$\\ \midrule
$\mathcal{L}_{\text{mae}}$  &
    13.0829 & 0.754  & 0.705  & 0.6065  \\
$\mathcal{L}_{\text{mse}}$  & 
    14.5819 & 0.6351 & 0.5886 & 0.4689 \\
$\mathcal{L}_{\text{mid}}$   & 
\textbf{32.778} & \textbf{-0.333} & \textbf{-0.418} & \textbf{-0.071} \\
\bottomrule
\end{tabular}
}
\label{tab:IFGSM}
\end{table}

\section{Hyperparameters of Attacks}
In Sec.~\textcolor{red}{5} in the main manuscript, 4 attack methods are utilized. For each attack method, there are hyperparameters which affect the strength of the attack.
Table~\ref{tab:attack_para} summarizes the chosen hyperparameters in tested attack methods in the main experiment, \ie, experiments in Sec.~\textcolor{red}{5}.

\begin{table}[htbp]
\caption{Hyperparameters of attacks}
\label{tab:attack_para}
\small
\centering
\begin{threeparttable}
\begin{tabular}{ll}
\toprule

Method & Hyperparameters\\ \midrule
FGSM & one step, $\epsilon=0.005,\alpha=0.01$ \\ 
Perceptual Attack & constraint: SSIM, weight =1,000,000 \\ 
UAP$^{*}$ & scale $=0.04$ \\ 
Kor.$^{*}$ Attack & learning rate: $0.2$ \\ 
\bottomrule
\end{tabular}
\end{threeparttable}
\end{table}

The meaning of these hyperparameters is explained in the original papers of attacks: FGSM~\cite{2015_ICLR_Goodfellow_FGSM}, Perceptual attack~\cite{2022_NIPS_Zhang_PAttack}, UAP~\cite{2022_BMVC_Ekarerina_UAP} and Kor. attack~\cite{2022_QEVMAw_Korhonen_BIQA}.

\section{Further Robustness Analysis}
In this section, we will further analyze the effectiveness of the NT strategy in improving the robustness of NR-IQA models.
In Sec.~\ref{sec:supp_KROCC_PLCC}, we present the robustness of baseline models and their NT-enhanced versions measured by KROCC, PLCC, and $R$ robustness~\cite{2022_NIPS_Zhang_PAttack}.
In Sec.~\ref{sec:supp_averaged}, we report the average metrics of RMSE and SROCC improvement for both baseline models and their NT-enhanced versions.
For each model, we provide the scatter plots of predicted scores before and after the perceptual attack in Sec.~\ref{sec:supp_scatter}, which intuitively show the effectiveness of the NT strategy.

\subsection{Robustness in Terms of KROCC, PLCC and $R$ Robustness}
\label{sec:supp_KROCC_PLCC}
In Sec.~\textcolor{red}{5.2} in the main manuscript, the robustness performances in terms of RMSE and SROCC are reported. 
Table~\ref{tab:KROCC}, Table~\ref{tab:PLCC} and Table~\ref{tab:R} show the robustness performances of NR-IQA models in terms of KROCC, PLCC, and $R$ robustness against different attack methods, respectively. 
Specifically, We evaluate $R$ robustness on four baseline methods as well as their NT-trained models with $\beta_1=100, \beta_2=0$.

In Table~\ref{tab:KROCC}, NR-IQA models with the NT strategy outperform their baseline models under all attacks when KROCC is measured between predicted scores before and after attacks. Among them, HyperIQA+NT witnesses a larger improvement in KROCC compared to its baseline under the FGSM attack, with KROCC increasing from $0.043$ of HyperIQA to $0.806$ using the NT strategy. Meanwhile, MANIQA demonstrates strong robustness against the Perceptual Attack, achieving a KROCC (scores before and after the attack) value of $1$. This means Perceptual Attack could not change the rank order of predicted scores before and after the attack on MANIQA. This phenomenon is also observed in the results of SROCC robustness.

\begin{table*}[htbp]
    \caption{The \textbf{KROCC$\uparrow$} metric of NR-IQA models against attacks (with ``baseline $\big /$ baseline+NT''). \textbf{Bold} denotes better value in a cell}
    \centering
    \renewcommand\arraystretch{1.3}
    \resizebox{\textwidth}{!}{
    \begin{tabular}{lcccc|cccc}
    \toprule
      &
      \multicolumn{4}{c}{MOS \& Predicted Score After Attack} &
      \multicolumn{4}{c}{Score Before Attack   \& Score After Attack} \\ \cmidrule(lr){2-5} \cmidrule(lr){6-9}
      & {\makecell[c]{HyperIQA \\ base / +NT }}            & {\makecell[c]{DBCNN \\ base / +NT }}    & {\makecell[c]{LinearityIQA \\ base / +NT }}    & {\makecell[c]{MANIQA \\ base / +NT }}  & {\makecell[c]{HyperIQA \\ base / +NT }}            & {\makecell[c]{DBCNN \\ base / +NT }}    & {\makecell[c]{LinearityIQA \\ base / +NT }}    & {\makecell[c]{MANIQA \\ base / +NT }}\Bstrut\\ 
      \hline
    FGSM &
    0.020 $\big/$ \textbf{0.610} &
    -0.146 $\big/$ \textbf{0.136} &
    -0.197 $\big/$ \textbf{-0.184} &
    {0.296} $\big/$ \textbf{0.584} &
    0.043 $\big/$ \textbf{0.806} &
    -0.071 $\big/$ \textbf{0.217} &
    \textbf{-0.156} $\big/$ {-0.171} &
    0.332 $\big/$  \textbf{0.749}\\

    Perceptual & 
    0.627 $\big/$ \textbf{0.669} &
    -0.079 $\big/$ \textbf{0.471}    &
    0.350 $\big/$  \textbf{0.415}   &
    0.870  $\big/$  \textbf{0.876}   &
      0.837 $\big/$ \textbf{0.997}&
      -0.091 $\big/$  \textbf{0.628}   &
      0.440 $\big/$  \textbf{0.566}   &
       \textbf{1.000} $\big/$ \textbf{1.000} \\
    
    UAP$^{*}$ &
    0.548 $\big/$ \textbf{0.628} &
    0.510 $\big/$\textbf{0.568}& 
    0.526 $\big/$ \textbf{0.543}& 
    0.578 $\big/$ \textbf{0.651}& 
    0.634 $\big/$\textbf{0.797}& 
    0.643 $\big/$ \textbf{0.708}& 
    0.664 $\big/$ \textbf{0.694}& 
    0.766 $\big/$ \textbf{0.871}\\ 
    
    Kor.$^{*}$ & 
    0.614 $\big/$\textbf{0.615} & 
    \textbf{0.678}$\big/${0.669} & 
    0.585 $\big/$\textbf{0.587}& 
    0.637 $\big/$\textbf{0.658}& 
    0.724 $\big/$\textbf{0.777}& 
    0.874 $\big/$\textbf{0.895}& 
    0.777 $\big/$\textbf{0.786}& 
    0.790 $\big/$ \textbf{0.850}\\ 
    \bottomrule
    \end{tabular}
    }
    \label{tab:KROCC}
    \end{table*}

\begin{table*}[htbp]
    \caption{The \textbf{PLCC$\uparrow$} metric of NR-IQA models against attacks (with ``baseline $\big /$ baseline+NT''). \textbf{Bold} denotes better value in a cell}
    \centering
    \renewcommand\arraystretch{1.3}
    \resizebox{\textwidth}{!}{
    \begin{tabular}{lcccc|cccc}
    \toprule
      &
      \multicolumn{4}{c}{MOS \& Predicted Score After Attack} &
      \multicolumn{4}{c}{Score Before Attack   \& Score After Attack} \\ \cmidrule(lr){2-5} \cmidrule(lr){6-9}
      & {\makecell[c]{HyperIQA \\ base / +NT }}            & {\makecell[c]{DBCNN \\ base / +NT }}    & {\makecell[c]{LinearityIQA \\ base / +NT }}    & {\makecell[c]{MANIQA \\ base / +NT }}  & {\makecell[c]{HyperIQA \\ base / +NT }}            & {\makecell[c]{DBCNN \\ base / +NT }}    & {\makecell[c]{LinearityIQA \\ base / +NT }}    & {\makecell[c]{MANIQA \\ base / +NT }} \Bstrut\\ 
      \hline
    FGSM &
      -0.009 $\big/$ \textbf{0.801} &
      -0.383 $\big/$ \textbf{0.251} &
      -0.497 $\big/$ \textbf{-0.387} &
        {0.599} $\big/$ \textbf{0.861} &
      0.042 $\big/$ \textbf{0.926} &
      -0.418 $\big/$ \textbf{0.196} &
      -0.569 $\big/$ \textbf{-0.439} &
       0.535 $\big/$  \textbf{0.929}\\
    Perceptual & 
    0.830 $\big/$ \textbf{0.868}&
    -0.030 $\big/$ \textbf{0.585}    &
    0.487$\big/$  \textbf{0.528}   &
    \textbf{0.696} $\big/$  0.691   &
    0.937 $\big/$  \textbf{1.000}   &
    -0.005$\big/$  \textbf{0.719}   &
    0.522 $\big/$  \textbf{0.582}   &
    \textbf{0.998} $\big/$ 0.995 \\
    
    UAP$^{*}$ &
    0.733 $\big/$ \textbf{0.817} &
    0.701 $\big/$ \textbf{0.776}& 
    0.694 $\big/$ \textbf{0.729}& 
    {0.766} $\big/$ \textbf{0.837}& 
    0.826 $\big/$ \textbf{0.943}& 
    0.811 $\big/$ \textbf{0.884}&  
    0.805 $\big/$ \textbf{0.876}& 
    0.928 $\big/$ \textbf{0.978}\\ 
    
    Kor.$^{*}$ & 
    0.801 $\big/$ \textbf{0.806} & 
    \textbf{0.875} $\big/$ 0.868 & 
    \textbf{0.774} $\big/$ \textbf{0.774}& 
    0.838 $\big/$ \textbf{0.856}& 
    0.875 $\big/$ \textbf{0.933}& 
    0.972 $\big/$ \textbf{0.980}& 
    0.914 $\big/$ \textbf{0.933}& 
    0.942 $\big/$ \textbf{0.969}\\ 
    \bottomrule
    \end{tabular}
    }
    \label{tab:PLCC}
\end{table*}

In Table~\ref{tab:PLCC}, NR-IQA models with the NT strategy perform better than their baseline models in most cases when PLCC is measured between predicted scores before and after attacks. For example, when the attack method is the Percepural Attack, the PLCC of DBCNN is -0.005 while the PLCC of DBCNN+NT is 0.719. The only exception is MANIQA where MANIQA+NT performs worse than MANIQA when attacked by the Perceptual Attack. This trend is consistent with the results reported in RMSE robustness. 

From Table~\ref{tab:KROCC} and Table~\ref{tab:PLCC}, we can conclude that the robustness of NR-IQA models in terms of RMSE and PLCC have similar trends, while robustness in terms of SROCC and KROCC show similar patterns. Additionally, the robustness improvement caused by the NT strategy is more obvious when NR-IQA models are attacked in white-box scenarios than in black-box scenarios.

From Table~\ref{tab:R}, we can see that NR-IQA methods with our NT strategy generally perform better than their baselines. However, it's essential to note that the definition of $R$ robustness assigns a higher weight to images with extremely large scores (close to $\beta_1$) or extremely small scores (close to $\beta_2$), whereas RMSE treats each image equally. Consequently, in scenarios where the DBCNN model is attacked by Kor.~attack, the NT model shows improvement in the RMSE metric but a decrease in the $R$ robustness compared to its baseline. Similar trends are observed when the LinearityIQA model is attacked by UAP.
Although different metrics focus on different aspects, the proposed NT strategy improves all robustness metrics in most cases.

\begin{table}[tbp]
\caption{The $R\uparrow$ robustness of NR-IQA models against attacks (with ``baseline $\big /$ baseline+NT'')}
\centering
\renewcommand\arraystretch{1.2}
\resizebox{\hsize}{!}{
\begin{tabular}{lcccc}
\toprule
       & {\makecell[c]{HyperIQA \\ base / +NT }}            & {\makecell[c]{DBCNN \\ base / +NT }}    & {\makecell[c]{LinearityIQA \\ base / +NT }}    & {\makecell[c]{MANIQA \\ base / +NT }} \\ \midrule
FGSM  & {0.659} $\big/$ \textbf{1.099} &
{0.328} $\big/$ \textbf{0.671} &
{1.011} $\big/$ \textbf{1.389} &
2.957   $\big/$    \textbf{3.864}    \\

Perceptual & {2.249}$\big/$ \textbf{3.047}&
{0.938} $\big/$ \textbf{2.076}  &
{1.011} $\big/$ \textbf{1.398}   &
\textbf{5.492}$\big/$   4.784  \\

UAP$^{*}$  & {1.180} $\big/$ \textbf{1.285}&
{1.054} $\big/$ \textbf{1.067}  &
\textbf{1.161} $\big/$ 1.096   &
{3.459}  $\big/$   \textbf{3.464} \\

Kor.$^{*}$ & {0.980} $\big/$ \textbf{1.092}&
\textbf{1.333} $\big/$ 1.323  &
{0.883} $\big/$ \textbf{1.000}   &
3.299  $\big/$  \textbf{3.319} \\
\bottomrule
\end{tabular}
}
\label{tab:R}
\end{table}

\begin{figure*}[htbp]
    \centering
    \includegraphics[width=\textwidth]{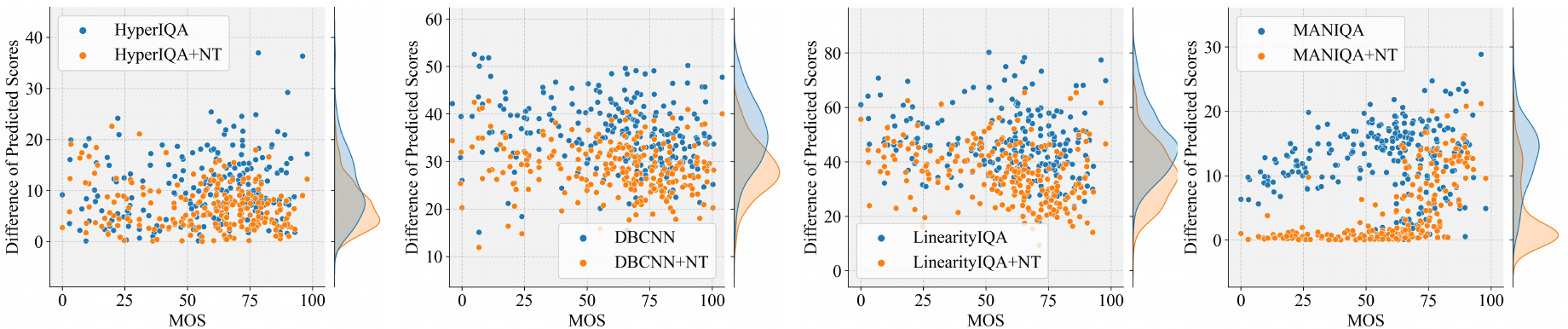}
    \caption{Comparison of four NR-IQA models with/without the NT strategy under the FGSM attack~\cite{2015_ICLR_Goodfellow_FGSM}. The absolute differences between predicted scores before and after attack for all test images are presented, with the fitted distribution displayed on the right side.}
    \label{fig:supp_scatter_FGSM}
\end{figure*}

\begin{figure*}[htbp]
    \centering
    \vspace{1.1em}
    \includegraphics[width=\textwidth]{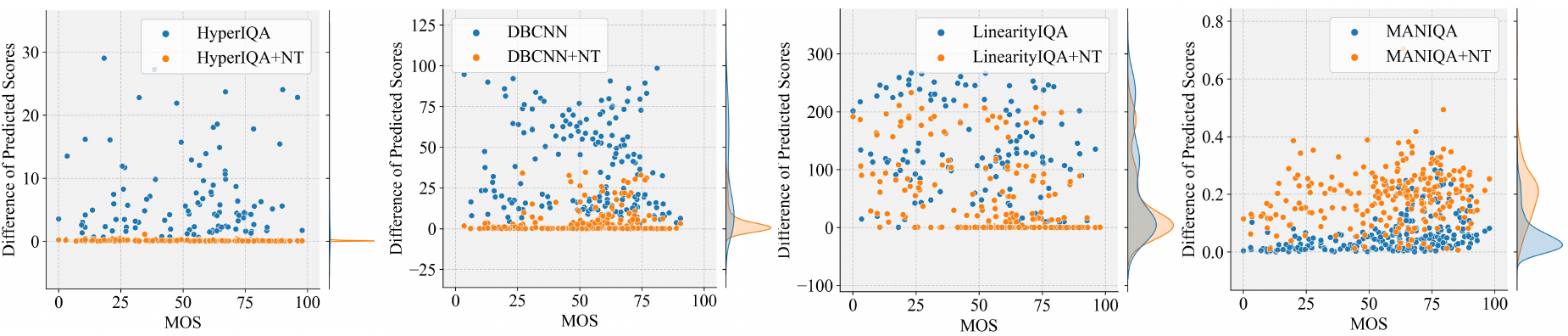}
    \caption{Comparison of four NR-IQA models with/without the NT strategy under the Perceptual attack~\cite{2022_NIPS_Zhang_PAttack}. The absolute differences between predicted scores before and after attack for all test images are presented, with the fitted distribution displayed on the right side.}
    \label{fig:supp_scatter_perceptual}
\end{figure*}

\begin{figure*}[htbp]
    \centering
    \includegraphics[width=\textwidth]{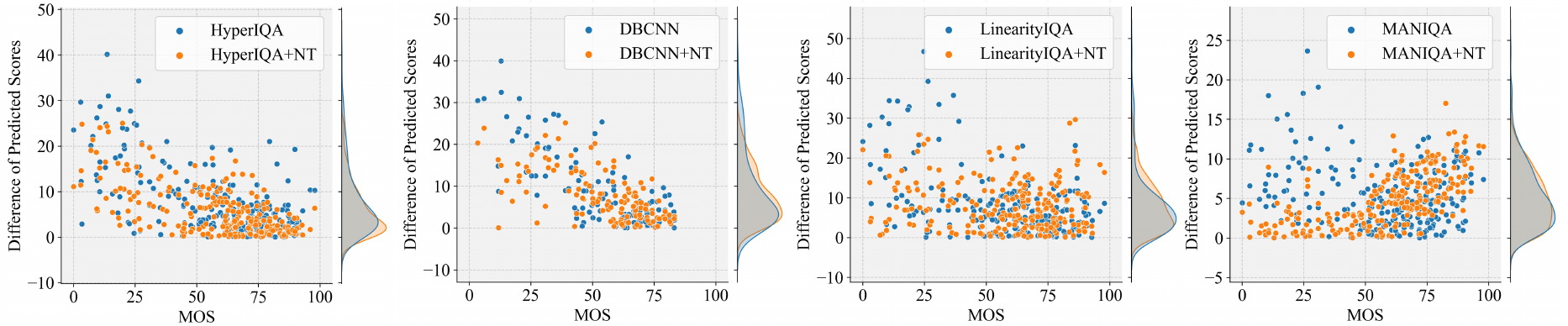}
    \caption{Comparison of four NR-IQA models with/without the NT strategy under the UAP attack~\cite{2022_BMVC_Ekarerina_UAP}. The absolute differences between predicted scores before and after attack for all test images are presented, with the fitted distribution displayed on the right side.}
    \label{fig:supp_scatter_UAP}
\end{figure*}

\begin{figure*}[htbp]
    \centering
    \vspace{2.5em}
    \includegraphics[width=\textwidth]{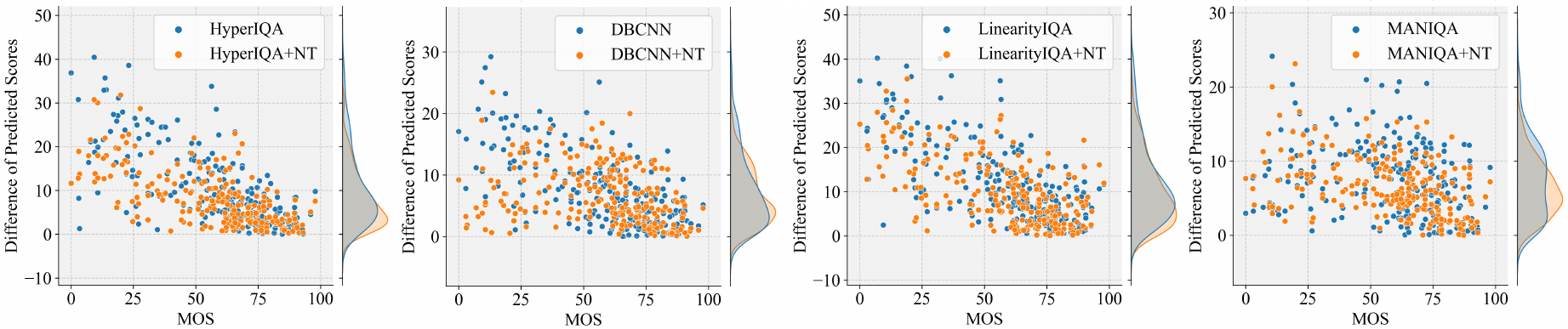}
    \caption{Comparison of four NR-IQA models with/without the NT strategy under the Kor.~attack~\cite{2022_QEVMAw_Korhonen_BIQA}. The absolute differences between predicted scores before and after attack for all test images are presented, with the fitted distribution displayed on the right side.}
    \label{fig:supp_scatter_Kor}
\end{figure*}

\subsection{Averaged Metrics of RMSE and SROCC Improvement}\label{sec:supp_averaged}
For an NR-IQA model subjected to an attack method, we calculate the difference in RMSE (or SROCC) between its NT version and the original model, denoted as $\Delta$RMSE (or $\Delta$SROCC). We then average $\Delta$RMSE and $\Delta$SROCC $\uparrow$ for both white-box and black-box attacks, as shown in Table~\ref{ReTab:overall_metric}, to corroborate \emph{Observation 4} presented in the main manuscript. These results further confirm the effectiveness of NT in mitigating both white-box and black-box attacks. %

\begin{table}[tbp]
\caption{Averaged $\Delta$RMSE $\downarrow$/ averaged $\Delta$SROCC $\uparrow$.}
\centering
\renewcommand\arraystretch{1.2}
\resizebox{\hsize}{!}{
\begin{tabular}{lcccc}
\toprule 
       & {HyperIQA}            & {DBCNN}    & {LinearityIQA}    & {MANIQA} \\ [-0.4ex] \midrule
White  &  -8.7595 /	0.4800 &	-31.5900 / 0.7465 &	-23.0075 /	0.0730 &	-4.4385	/ 0.2250\\[-0.4ex]
Black  &  -3.0215 /	0.0730 &	-2.5635 / 0.0280	&	-1.8895	/ 0.0165 &	-0.6410 / 0.0400	\\[-0.4ex]
Overall  &  -5.8905 / 0.2765	&	-17.0768 / 0.3873	&	-12.4485 / 0.0448	&	-2.5398 / 0.1325 \\[-0.4ex]			
\bottomrule
\end{tabular}
}
\label{ReTab:overall_metric}
\end{table}

\subsection{Distributions of Predicted Scores}
\label{sec:supp_scatter}

Figure~\ref{fig:supp_scatter_FGSM}--\ref{fig:supp_scatter_Kor} illustrate the absolute differences between predicted scores before and after various attacks for all test images (from the first row to the last row: FGSM, Perceptual Attack, UAP, and Kor. attack). The fitted distribution is presented on the right side of each image.

It is evident that all models trained with the NT strategy exhibit smaller score changes compared to their corresponding baseline models. Additionally, we observe an interesting trend: the NT strategy enhances robustness for different NR-IQA models at various image quality levels.

For example, considering the Perceptual Attack, all points for HyperIQA+NT closely align with the line ``difference of predicted scores = 0''. This highlights the significant effectiveness of the NT strategy in minimizing score changes with small perturbations for HyperIQA. For the DBCNN model, it is clear that the NT strategy brings about more reduction in score changes for images with MOS between $[0, 50)$ and $[75, 100]$. 

Conversely, in the case of LinearityIQA, the effectiveness of the NT strategy is more obvious on high-quality images with MOS in $[75, 100]$, while it proves more effective on low-quality images with MOS in $[0, 50]$ for MANIQA.
This discovery reflects that the NT strategy has varying impacts on images with different quality levels, and these impacts are closely tied to the NR-IQA models.
Exploring the enhancement of adversarial robustness in NR-IQA models across different image quality levels represents a valuable avenue for research. Such investigations can shed light on the properties of NR-IQA models in predicting scores for images of differing quality.

\begin{figure*}[!t]
    \centering
    \includegraphics[width=0.8\textwidth]{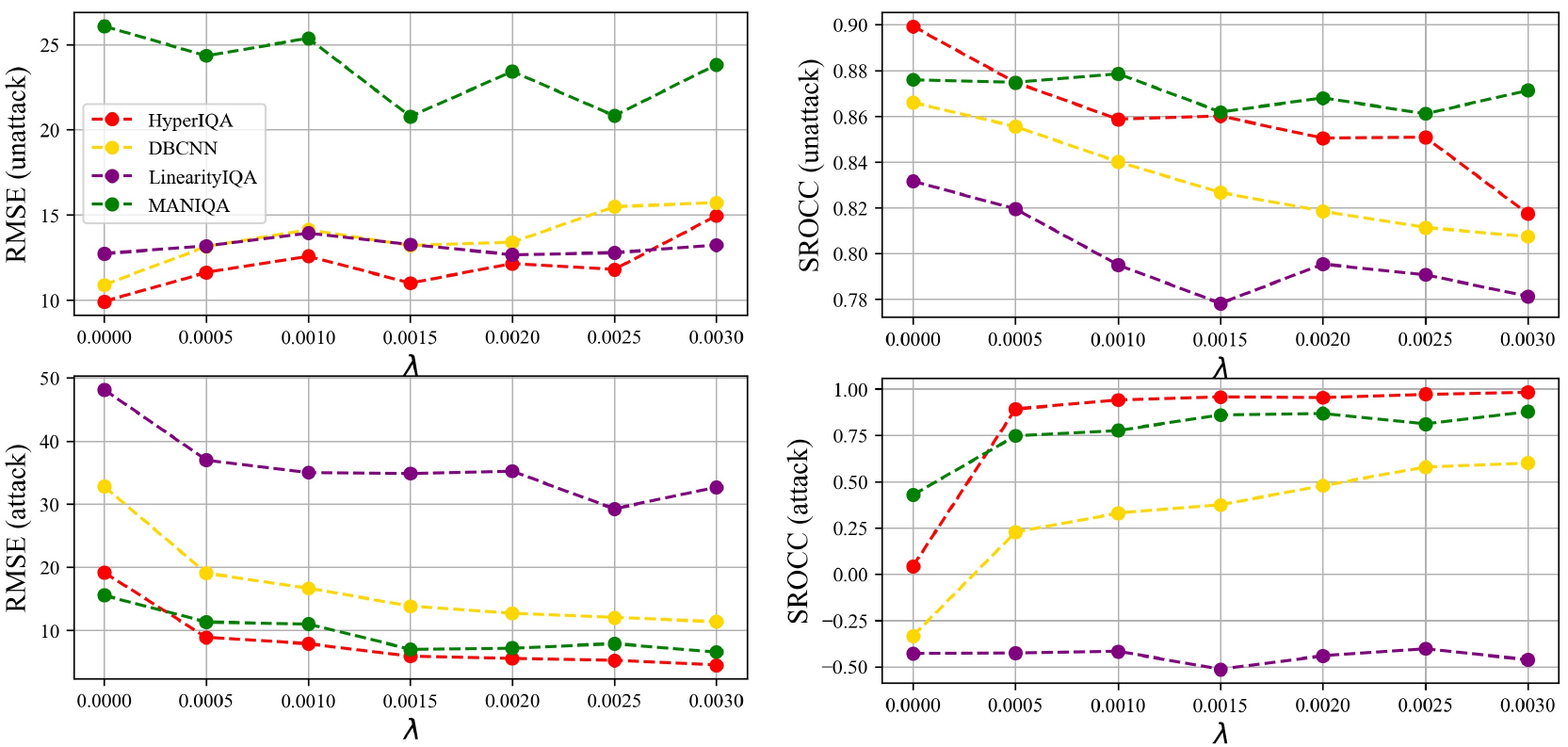}
    \caption{The impact of $\lambda$ to SROCC and RMSE on both unattacked images and adversarial examples.}
    \label{fig:supp_lambda}
\end{figure*}

\section{Full Results of Ablation Studies}
In Figure~\ref{fig:supp_lambda}, we show the full results of the ablation study of $\lambda$ (mentioned in Sec.~\textcolor{red}{5.5} in the main manuscript). Our analysis focuses on two aspects of an NR-IQA model: its performance on unattacked images and its robustness against attacks. For the former, we utilize SROCC on unattacked images across MOS values and predicted scores, and for the latter, we employ the RMSE between predicted scores before and after the FGSM attack. 

As $\lambda$ increases, the performance of baseline+NT models has the following trend on unattacked images. SROCC values generally decrease with the rising $\lambda$, while RMSE values exhibit an upward trend (except for MANIQA). It is an interesting observation that the RMSE value of MANIQA fluctuates as $\lambda$ changes, and the RMSE tends to decrease with larger $\lambda$.
When attacked by the FGSM attack, the RMSE values of all baseline+NT models decrease consistently with the increase of $\lambda$. Except for LinearityIQA, the SROCC values of other models increase as $\lambda$ becomes larger. This implies that increasing $\lambda$ tends to enhance the robustness of NR-IQA models but leads to a performance decline on unattacked images.

\section{Visualization Results} 
In this section, we present visualization results to illustrate the effectiveness of the NT strategy under FGSM attack with different attack intensities and UAP. Under FGSM attack with different attack intensities, for each pair of baseline and baseline+NT models, we provide two sets of visualization results: one for high-quality images and the other for low-quality images. We show the normalized MOS of the original image. Under the UAP attack, we provide one adversarial sample for an NR-IQA model.  We display adversarial examples for both the baseline model and the baseline+NT model, along with the corresponding score changes (predicted score before attack $\rightarrow$ predicted score after attack). 

\begin{figure*}[!thb] 
    \centering
    \includegraphics[width=0.76\linewidth]{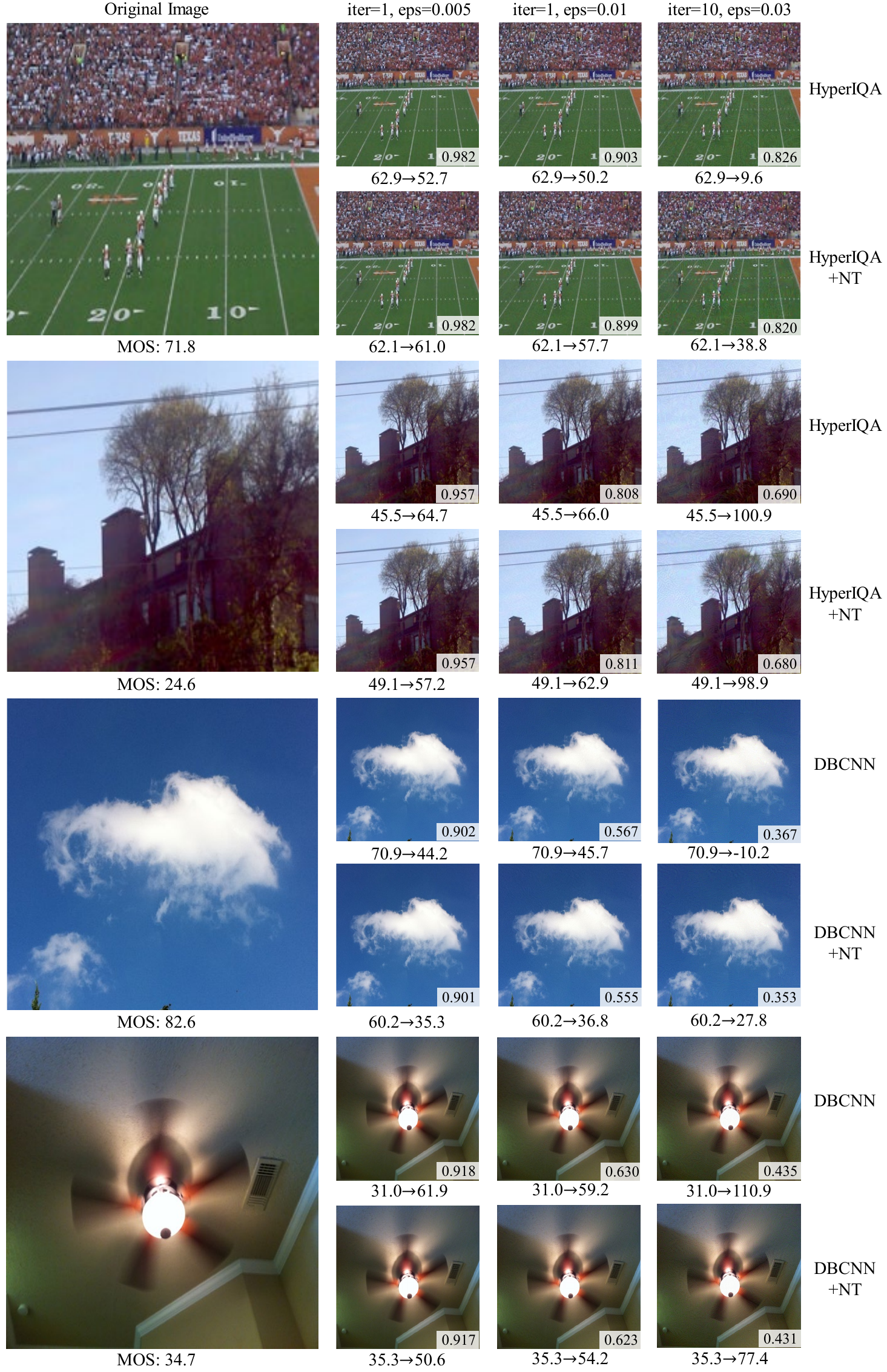}
    \caption{(Zoom in for a better view) Visualization results of adversarial examples generated using the FGSM with different intensities. The normalized MOS is presented. FGSM attack settings are indicated above the figures, and each adversarial example for a model is presented with the format: ``predicted score before attack $\rightarrow$ predicted score after attack'' below the respective images. The SSIM between the adversarial example and the original image is displayed at the bottom right corner of each adversarial image.}
    \label{fig:visual_1}
\end{figure*}

\begin{figure*}[!thb] 
    \centering
    \includegraphics[width=0.78\linewidth]{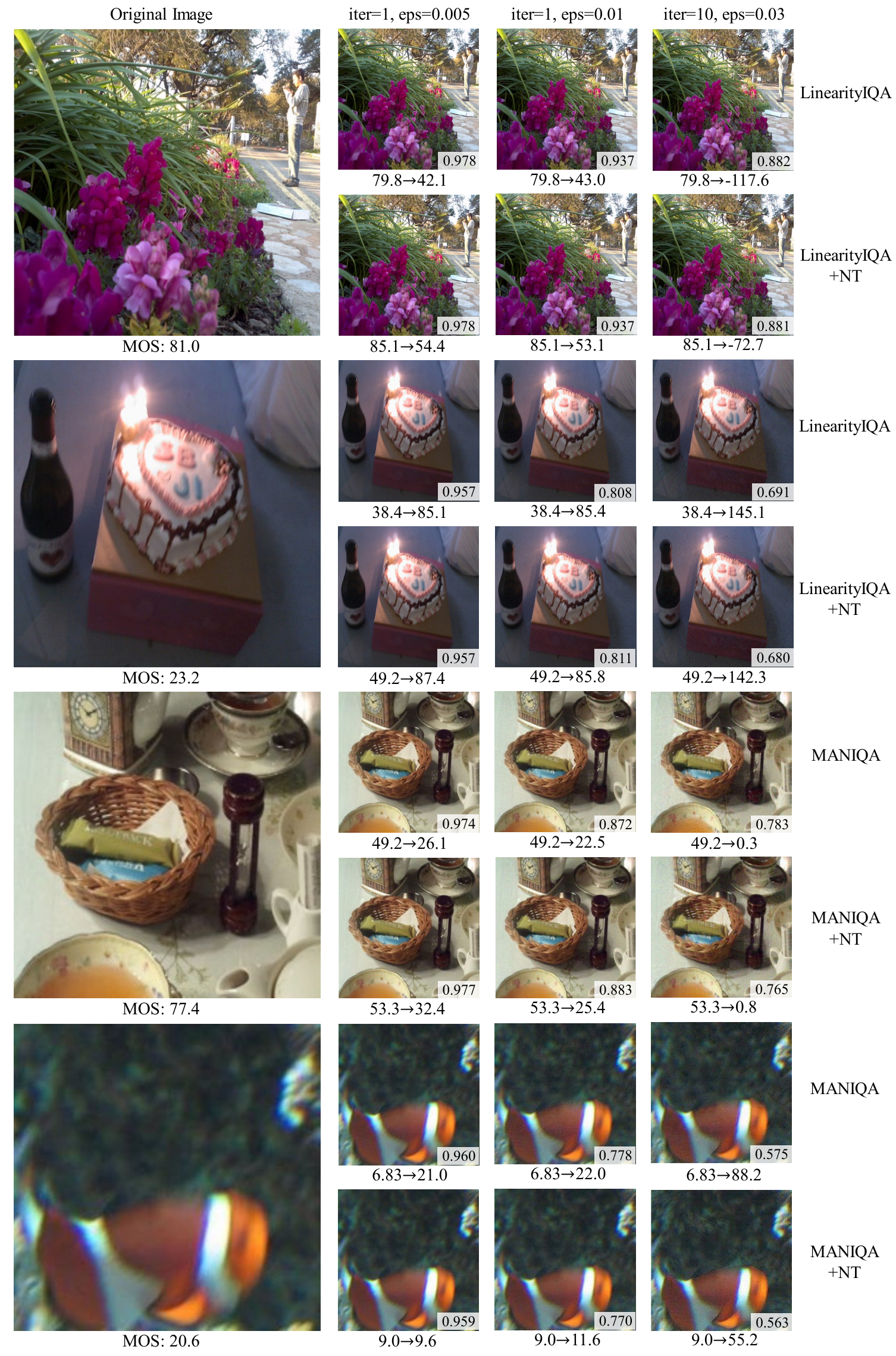}
    \caption{(Zoom in for a better view) Visualization results of adversarial examples generated using the FGSM with different intensities. The normalized MOS is presented. FGSM attack settings are indicated above the figures, and each adversarial example for a model is presented with the format: ``predicted score before attack $\rightarrow$ predicted score after attack'' below the respective images. The SSIM between the adversarial example and the original image is displayed at the bottom right corner of each adversarial image.}
    \label{fig:visual_2}
\end{figure*}

\begin{figure*}[!thb] 
    \centering
    \includegraphics[width=0.78\linewidth]{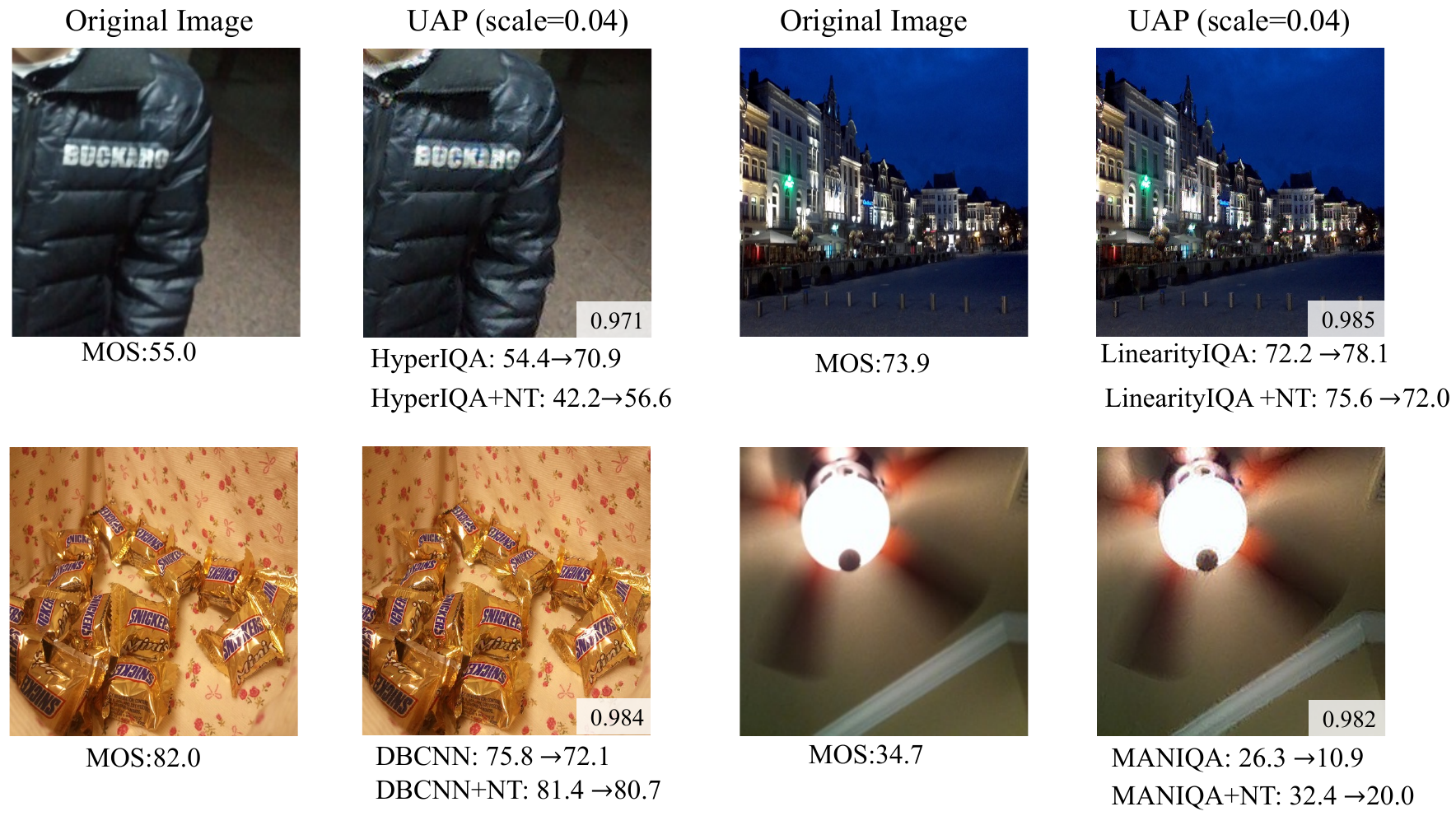}
    \caption{(Zoom in for a better view) Visualization results of adversarial examples generated using the UAP attack. The normalized MOS is presented. Each adversarial example for a model is presented with the format: ``predicted score before attack $\rightarrow$ predicted score after attack'' below the respective images. The SSIM between the adversarial example and the original image is displayed at the bottom right corner of each adversarial image.}
    \label{fig:visual_uap}
\end{figure*}

Figure~\ref{fig:visual_1} shows visualization results of FGSM attack for HyperIQA, DBCNN, and their NT versions. Figure~\ref{fig:visual_2} displays visualization results of FGSM attack for LinearityIQA, MANIQA, and their NT versions.  Figure~\ref{fig:visual_uap} shows visualization results of UAP attack for HyperIQA, DBCNN, LinearityIQA, MANIQA, and their NT versions.

From Figure~\ref{fig:visual_1} and Figure~\ref{fig:visual_2}, we can see that the imperceptibility of adversarial perturbations for images gets worse as the attack intensity increases. However, despite this, the NT models consistently exhibit smaller score changes than the baseline models in most cases. Consider HyperIQA and its NT version as an example. When attacked by the strongest FGSM attack (iter=10, eps=0.01), the score change for HyperIQA+NT on the high-quality image is $62.1-38.8=23.3$, whereas the score change for HyperIQA is $62.9-9.6=53.3$. Similarly, for the low-quality image, the score change for the NT version is $49.8$, while the change for the baseline model is $55.4$. From Figure~\ref{fig:visual_uap}, we can see that with the same adversarial sample, baseline and their NT versions have different defense performances. For example, the score change for HyperIQA+NT on its adversarial sample is $56.6-42.2=14.4$, whereas
HyperIQA on the same adversarial sample is $70.9-54.4=16.5$.

\end{document}